%% file: main.tex
\documentclass[twoside,11pt]{article}

\usepackage{jmlr2e}

\input{defs}

\usepackage{times}
\usepackage{amsmath}
\usepackage{tabularx}
\usepackage{pifont}
\usepackage{xcolor}
\usepackage{placeins}
\newcommand{\cmark}{\textcolor{green}{\ding{51}}}
\newcommand{\xmark}{\textcolor{red}{\ding{55}}}

\newcommand{\hy}{\hat y}
\newcommand{\hx}{\hat x}
\newcommand{\hell}{\hat \ell}
\newcommand{\by}{\bar y}

\newcommand{\bell}{\bar \ell}
\newcommand{\hEE}{\hat \EE}

\usepackage{cleveref}
\usepackage{thmtools}
\usepackage{thm-restate}
\usepackage{Definitions}

\firstpageno{1}

\author{\name Hieu Pham${}^{*}$ \email hyhieu@google.com \\
\name Zihang Dai${}^{*}$ \email zihangd@google.com \\
\name Golnaz Ghiasi${}^{*}$ \email golnazg@google.com \\
\name Kenji Kawaguchi${}^{*}$ \email kawaguch@csail.mit.edu \\
\name Hanxiao Liu \email hanxiaol@google.com \\
\name Adams Wei Yu \email adamsyuwei@google.com \\
\name Jiahui Yu \email jiahuiyu@google.com \\
\name Yi-Ting Chen \email yitingchen@google.com \\
\name Minh-Thang Luong \email thangluong@google.com \\
\name Yonghui Wu \email yonghui@google.com \\
\name Mingxing Tan \email tanmingxing@google.com \\
\name Quoc V. Le \email qvl@google.com \\
\textit{*: Equal contributions.}\\\textnormal{Corresponding authors:} \{hyhieu,zihangd\}@google.com.
}

\editor{\textit{To be assigned.}}

\hypersetup{
  pdftitle = {Combined Scaling for Zero-shot Transfer Learning},
}

\newcommand{\ours}{BASIC}

\newcommand{\gain}[1]{{\textcolor[rgb]{0.05859375,0.61328125,0.34375}{\textbf{(+#1)}}}}
\newcommand{\loss}[1]{{\textcolor[rgb]{1.,0.243137255,0.188235294}{\textbf{(#1)}}}}

\title{Combined Scaling for Zero-shot Transfer Learning}

\begin{document}

\maketitle
\begin{abstract}

We present a combined scaling method -- named BASIC -- that achieves 85.7\% top-1 accuracy on the ImageNet ILSVRC-2012 validation set without learning from any labeled ImageNet example.
This accuracy surpasses best-published similar models -- CLIP and ALIGN -- by 9.3\%.
Our BASIC model also shows significant improvements in robustness benchmarks.
For instance, on 5 test sets with natural distribution shifts such as ImageNet-\{A,R,V2,Sketch\} and ObjectNet, our model achieves 84.3\% top-1 average accuracy, only a small drop from its original ImageNet accuracy.

To achieve these results, we scale up the contrastive learning framework of CLIP and ALIGN in three dimensions: data size, model size, and batch size.
Our dataset has 6.6B noisy image-text pairs, which is 4x larger than ALIGN, and 16x larger than CLIP.
Our largest model has 3B weights, which is 3.75x larger in parameters and 8x larger in FLOPs than ALIGN and CLIP.
Finally, our batch size is 65536 which is 2x more than CLIP and 4x more than ALIGN. 

We encountered two main challenges with the scaling rules of BASIC.
First, the main challenge with implementing the combined scaling rules of BASIC is the limited memory of accelerators, such as GPUs and TPUs.
To overcome the memory limit, we propose two simple methods which make use of gradient checkpointing and model parallelism. 
Second, while increasing the dataset size and the model size has been the defacto method to improve the performance of deep learning models like BASIC, the effect of a large contrastive batch size on such contrastive-trained image-text models is not well-understood.
To shed light on the benefits of large contrastive batch sizes, we develop a theoretical framework which shows that larger contrastive batch sizes lead to smaller generalization gaps for image-text models such as BASIC.

\end{abstract}
\newpage

\tableofcontents

\newpage

\section{\label{sec:intro}Introduction}
The recent advances in multimodal training approaches such as CLIP~\citep{radford21learning} and ALIGN~\citep{jia21scaling}
have the potential to eliminate the need for collecting labeled training data for every new application.
Using natural language as a weak supervision signal, CLIP and ALIGN achieve the impressive top-1 accuracy of 76.2\% and 76.4\% on ImageNet ILSVRC-2012 without learning from any labeled ImageNet data.
In addition to the promising accuracy on ImageNet, the so-called ``zero-shot'' models in CLIP and ALIGN demonstrate two important properties.
First, these models are versatile, as they can be directly deployed on many downstream tasks without task-specific data for finetuning.
Second, CLIP and ALIGN models are  more robust than traditional classifiers.
Robustness evaluations on benchmarks with natural distribution shifts~\citep{hendrycks2019nae,hendrycks2020many,recht2019imagenet,barbu2019objectnet,wang2019learning} show that the accuracy of models like CLIP and ALIGN typically drops less than 10\%, while the accuracy of supervised and semi-supervised models might drop as much as 40\%~\citep{taori2020measuring,szegedy2013intriguing}.

Despite their versatility and robustness, the best models from CLIP and ALIGN are still not as competitive as supervised and semi-supervised models when enough labeled data is available, which can limit their potential applications.
For example, the best CLIP and ALIGN models have an accuracy around 76\% on ImageNet, which is only comparable with a supervised ResNet-50~\citep{he15deep}, and significantly worse than the state-of-the-art supervised training on ImageNet (without extra data: 87.1\%~\citep{yuan2021volo}, and with extra data: 90.88\%~\citep{dai21coatnet}).
Therefore, narrowing the gap from these models to supervised and semi-supervised models would make the image-text contrastive learning approach in CLIP and ALIGN a viable alternative for image classification.

In this paper, we develop significantly better image classifiers that leverage the image-text contrastive learning approaches like CLIP and ALIGN at a much larger scale.
In particular, we scale up the contrastive learning framework of CLIP~\citep{radford21learning} and ALIGN~\citep{jia21scaling} in 3 dimensions: dataset size, model size, and batch size.
For the data, we expand the ALIGN dataset~\citep{jia21scaling} from 1.7B noisy image-text pairs to 6.6B pairs, \ie,~almost 4x larger.
For the models, we choose CoAtNet, an architecture with higher learning capacity~\citep{dai21coatnet}, and scale it to 3B parameters,~\ie,~3.75x more weights and 8x more FLOPs than the largest models in CLIP and ALIGN.
For the batch size, we use 65536 contrastive learning examples per minibatch, \ie,~2x more than CLIP and 4x more than ALIGN.

\paragraph{Overview of our implementation.}
The fundamental bottleneck of training large models at larger batch sizes is the limited memory of deep learning accelerators such as GPUs and TPUs.
We propose two approaches that allow practitioners to overcome such memory limits.

Our first approach (Section~\ref{sec:pipeline_and_accum}) makes use of micro-batch pipelining~\citep{huang2019gpipe} and gradient accumulation (GradAccum)~\citep{ott2018scaling,zhai21scaling}.
Our second approach (Section~\ref{sec:spmd_sharding}) utilizes the model parallelism scheme of Single-Program Multi-Data (SPMD)~\citep{lepikhin2020gshard,xu2021gspmd} to distribute the weights of certain layers in our networks onto different devices.
While our SPMD approach is faster than our pipelining approach, and can deliver exact computations, the SPMD approach requires more manual designs to scale to arbitrarily large contrastive batch sizes, and hence, is less general than the pipelining approach.

Both our pipelining approach and our SPMD approach make use of \textit{gradient checkpointing}~\citep{chen2016training}, which is also called \textit{rematerialization} in certain literature~\citep{kumar2019efficient,jain2020checkmate}.
The idea behind rematerialization is to discard certain intermediate values in the forward pass of a neural network to save memory, and then recompute -- \ie,~rematerialize -- these values only when they are needed for gradient computation in the network's backward pass.

\paragraph{Overview of our theoretical insights.}
While the benefits of large datasets and large models for deep learning models have become established knowledge, the benefits of large batch size are less well-understood in the context of relatively new image-text contrastive models.
To understand such benefits, we develop a theoretical analysis of the image-text contrastive learning framework of CLIP and ALIGN.
Our analysis establishes that using a larger contrastive batch size in CLIP and ALIGN's framework leads to a smaller generalization gap of the resulting models.
% In particular, under regularity assumptions, the difference between the empirical contrastive training loss and the expected test loss is bounded by a term of the size $B^{-1/2}$

\begin{table}[h!]
\centering
\resizebox{0.75\linewidth}{!}{ %
\begin{tabular}{lccr@{\hskip0.01\linewidth}l}
  \toprule
  & \textbf{ALIGN}~\citep{jia21scaling}
  & \textbf{CLIP}~\citep{radford21learning}
  & \multicolumn{2}{@{\hskip 0.007\linewidth}c}{\textbf{BASIC} (ours)} \\
  \midrule
    ImageNet & 76.4 & 76.2 & \textbf{85.7} & \gain{9.3} \\
    ImageNet-A & 75.8 & 77.2 & \textbf{85.6} & \gain{8.4} \\
    ImageNet-R & 92.2 & 88.9 & \textbf{95.7} & \gain{3.5} \\
    ImageNet-V2 & 70.1 & 70.1 & \textbf{80.6} & \gain{10.5} \\
    ImageNet-Sketch & 64.8 & 60.2 & \textbf{76.1} & \gain{11.3} \\
    ObjectNet & 72.2 & 72.3 & \textbf{82.3} & \gain{10.1} \\
  \midrule
    Average & 74.5 & 74.2 & \textbf{84.3} & \gain{10.1} \\
  \bottomrule
\end{tabular}
}
\caption{\label{tab:highlights}Highlights of our key results. Shown are the top-1 accuracy of our method, BASIC, and similar baselines -- CLIP and ALIGN -- on ImageNet and other robustness test sets. None of these models have learned from any labeled training example in ImageNet. On average, BASIC surpasses the baselines by the significant \textcolor[rgb]{0.05859375,0.61328125,0.34375}{\textbf{10.1}} percentage points.}
\end{table}

\paragraph{Overview of our empirical results.}
Our proposed method, called \textbf{\ours}, for \textbf{B}atch, D\textbf{a}ta and Model \textbf{SI}ze \textbf{C}ombined Scaling,  achieves drastic improvements over CLIP and ALIGN models. 
For instance, on ImageNet, the largest BASIC model achieves 85.7\% top-1 accuracy, without learning from any labeled example in the ImageNet training set.
This result surpasses similar models in CLIP and ALIGN 9.3\%.
This BASIC model also shows significant improvements on robustness benchmarks. For instance, on 5 test sets with natural distribution shifts such as ImageNet-\{A,R,V2,Sketch\} and ObjectNet, the model achieves an average of 83.7\% top-1 accuracy, only a small drop from its original ImageNet accuracy (see Table~\ref{tab:highlights}).
When tested against CLIP on the other 17 image classification benchmarks, \eg,~CIFAR, Caltech101, Flowers, etc. BASIC outperforms CLIP on 13 out of these 17 benchmarks.\looseness=-1

%On average, our model achieves 86.4\% accuracy whereas CLIP achieves 79.7\% accuracy. For example, on the ImageNet ILSVRC-2012 validation set, our model achieves 85.7\% top-1 accuracy, which is 9.5\% better than CLIP and ALIGN. Interestingly, as BASIC models become more accurate on ImageNet ILSVRC-2012, they also become substantially more robust against natural distribution shifts. For instance, the average accuracy of BASIC models on 5 robustness benchmarks are 83.7\%\textcolor{red}{(85.3\% in abstract, pls check)}, which is only slightly lower than its accuracy of 85.7\% on ILSVRC-2012.

\section{\label{sec:related}Related Work}

\paragraph{Large-scale pretraining and the contrastive loss.} %Training large models on large datasets is a common recipe that leads to many impressive results. 
As computer vision models grow in their size and capacity, many weakly-supervised and self-supervised pretraining methods have been proposed to learn good visual representations.
On one hand, pretraining with a classification loss on large weakly-labeled datasets such as Instagram hashtags or JFT can produce significant gains on downstream tasks such as ImageNet~\citep{joulin2016learning,mahajan2018explore,kolesnikov20bit,dosovitskiy21vit,sun2017revisiting,zhai21scaling}.
%Despite the success of these weakly-labeled datasets, they are still expensive to collect and many of them are proprietary.
%As such, 
On the other hand, self-supervised methods which leverage existing structures in unlabeled data to train models have been developed.
A promising development in self-supervised learning is the contrastive loss, with representative works like CPC~\citep{aaron18cpc}, SimCLR~\citep{chen20simclr,chen2020big} and MoCo~\citep{he20moco,chen2020improved}.
In this paper, we scale up the contrastive learning framework, which we will revisit in  detail in Section~\ref{sec:background}.

\paragraph{Contrastive-learned image-text models.} Unlike the single-modal contrastive approaches mentioned in the previous paragraph, our work leverages data from two modalities: image and text.
Using images with accompanying text is related to the literature on image-captioning models, such as~\citep{vinyals2015show,karpathy2015deep,xu2015show,joulin2016learning,li2017learning,sariyildiz2020learning,zhang2020contrastive,desai2021virtex}.
While learning to generate captions from images can induce good visual representations, it is not the goal of this paper.
Instead, this paper focuses on establishing the ability of models to \textit{classify images based on textual descriptions}.
This focus makes our work closely related to the recent work of image-text models such as CLIP~\citep{radford21learning} and ALIGN~\citep{jia21scaling}.
Similar to CLIP and ALIGN, our work also learns the mapping between images and texts, which is related to many previous works, such as~\citep{mori1999,weston2010,socher2010connecting,socher2013zero,hodosh2013framing,frome2013devise,norouzi2013zero,kiros2014unifying,socher-etal-2014-grounded,akata2015evaluation,akata2015label,nam2017dual,faghri2017vse++,li2019visual,liu2019aligning,lu2019vilbert,messina2020fine,chen2020uniter,huang2020pixel,chen2021learning}.

\paragraph{Differences between our work and zero-shot learning.}
Early works on zero-shot vision models date back to the 2000s, \eg,~\citep{larochelle08zero,zhang17zero,xian16zero,xian17zero,schonfeld19zero}.
In these works, the term ``zero-shot'' refers to the ability of models to \textit{``generalize to classes or tasks for which no training data are available and only a description of the classes or tasks are provided''}.
Under such definition, BASIC models -- as well as the recent work that BASIC is based on such as CLIP~\citep{radford21learning} and ALIGN~\citep{jia21scaling} -- are \textit{not} ``zero-shot learned'' models.
This is because the data curating procedures of BASIC, CLIP, and ALIGN can exposes certain class names to their models, albeit not intentionally.
For instance, when an image of a golden retriever dog is crawled from the internet, the image could come from a file named \texttt{my\_golden\_retriever.jpg} which was uploaded by a user.
If a model in BASIC, CLIP, or ALIGN learns to associate the content of such an image with the text sequence ``my golden retriever'' as parsed from the image's file name, and then the model uses the knowledge from such association at test time, then the model is not zero-shot.
Despite being \textit{not} zero-shot, models from BASIC, CLIP, and ALIGN retain their claimed benefits on versatility and robustness.

\paragraph{Zero-shot transfer learning} Instead of zero-shot learning, CLIP and ALIGN are known to conduct zero-shot \textit{transfer} learning  \citep{radford21learning,jia21scaling,zhai2022lit}. Zero-shot transfer learning differs significantly from  zero-shot learning. Unlike zero-shot learning, it permits relevant supervised information during pretraining, while it  allows no supervised examples
during the transfer protocol; i.e., zero-shot transfer learning skips  the finetuning stage completely and performs the downstream task  based only on a text description of the target classes. For example, see \citep{radford21learning,jia21scaling,zhai2022lit}
for more details on this terminology.

\paragraph{Data, model and batch scaling.}
Scaling has proven to be a powerful tool to boost the efficacy of vision model pretraining. There are three dimensions one can scale on. The simplest dimension is data. Indeed, recent efforts have shown that the more data we train on, the better the models become~\citep{joulin2016learning,mahajan2018explore,kolesnikov20bit,dosovitskiy21vit,sun2017revisiting}. The second dimension is the model size, with representative works such as EfficientNet, VITs and related works~\citep{tan19efficientnet,tan21efficientnetv2,tan20edet,dosovitskiy21vit,zhai21scaling,bello2021revisiting}. 
Lastly, scaling up batch sizes is also the key for improving the model effectiveness~\citep{goyal17onehour}, especially for the contrastive loss~\citep{chen20simclr,tian20view,jia21scaling,radford21learning}. 
% For example, \citep{goyal17onehour} scales up the batch size to 8192 with gradient aggregation technique, such that a ResNet-50 model trained on supervised ImageNet classification can finish within one hour. Recently, \citet{chen20simclr} shows significant performance gain with batch size 8192 on contrastive visual learning. Later on, the batch size is further scaled up to 65536 in~\citep{tian20view} to boost the performance. 
% The similar idea immediately shows up in the recent image-text pretraining models. For example, ALIGN~\citep{jia21scaling} adopts a batch size of 16384, while CLIP~\citep{radford21learning} even doubles it. 
Our work is inspired by the power of scaling, and pushes the limits in all the dimensions.

\newpage
\section{\label{sec:background}Background on Image-text Contrastive Learning and Zero-shot Transfer Learning}
In this section, we revisit the contrastive training framework for parallel image-text data, as introduced by CLIP~\citep{radford21learning} and ALIGN~\citep{jia21scaling}. In doing so, we define the notations that will be used throughout the remaining of this paper.

Let $\vx \in \mathcal{X}$ be an arbitrary image and $\vy \in \mathcal{Y}$ be an arbitrary text sequence.
The image-text contrastive training framework~\citep{radford21learning,jia21scaling} trains an image encoder $F$ and a text encoder $G$ to map $\vx$ and $\vy$ into a $D$-dimensional unit sphere,~\ie,~$F(\vx), G(\vy) \in \sS^{D}$.
The desiderata of these encoders is that images and text sequences of similar semantics should be mapped to nearby points in the latent space, while those with different semantics should be mapped to distant points in the space.
To train $F$ and $G$ to achieve such desiderata, a minibatch gradient training procedure is used.
At each step in this training procedure, $F$ and $G$ receives $B$ image-text pairs,~\ie~$(\vx_i, \vy_i)$ for $i = 1, 2, ..., B$.
Based on the embeddings computed by $F$ and $G$, a similarity matrix $\mA$ is computed, where $\mA_{i, j} = \big( F(\vx_i)^\top G(\vy_i) \big) / \tau$.
Here, $\tau$ is called the softmax temperature which serves to steepen or dampen the softmax distributions in the rows and columns of $\mA$.
From this similarity matrix $\mA$, two softmax-like losses are computed based on the rows and the columns of $\mA$:
\begin{align}
\text{RowLoss}_B &= -\frac{1}{B} \sum_{i=1}^{B} \log{\frac{\mA_{i, j}}{\sum_{k=1}^{B} \mA_{i, k}}}  \\
\text{ColumnLoss}_B &= -\frac{1}{B} \sum_{j=1}^{B} \log{\frac{\mA_{i, j}}{\sum_{k=1}^{B} \mA_{k, j}}} \\
\label{eqn:objective}
\text{ContrastiveLoss}_B &= \frac{\text{RowLoss}_B + \text{ColumnLoss}_B}{2}
\end{align}
Minimizing $\text{ContrastiveLoss}_B$ encourages the entries on the diagonal of $\mA$ to be large while the entries elsewhere to be small.
Equivalently, images and text sequences from the same pair in the minibatch,~\ie~$x_i$ and $\vy_i$, will be embedded into nearby points, while those from different pairs,~\ie~$x_i$ and $\vy_{j \neq i}$, will be embedded into distant points.
The resulting encoders $F$ and $G$ thus achieve the desiderata of the contrastive learning framework.

\section{\label{sec:pipeline_and_accum}Batch Size Scaling with Pipelining and Gradient Accumulation}
We start this section by discussing the memory bottleneck in the contrastive training framework as described in Section~\ref{sec:background}.
We focus on memory because it is the most crucial bottleneck which hinders two out of three dimensions that we want to scale,~\ie,~model size and batch size.
We further show that the vanilla pipelining algorithm~\citep{huang2019gpipe} with gradient accumulation (GradAccum)~\citep{ott2018scaling,zhai21scaling} is not directly applicable to contrastive learning.
We then describe our modifications to make GradAccum work for constrastive learning.

\subsection{The Memory Bottleneck for Image-Text Contrastive Learning}

\paragraph{The memory bottleneck.}
The consensus among representative work in contrastive learning~\citep{chen2020big,chen2020improved,he20moco,chen20simclr} is that the larger the networks trained with a larger contrastive batch size performs better.
This observation is further explained by our theoretical analysis in Section~\ref{sec:theory}, and is confirmed by empirical results in Section~\ref{sec:ablation}.
Therefore, we want to enlarge the networks $F$, $G$, \textit{and} the batch size $B$. However, this will create a memory bottleneck. 
Three well-known techniques to relieve memory burden are gradient accumulation (GradAccum) ~\citep{ott2018scaling,zhai21scaling}, re-materialization (or gradient checkpointing)~\citep{griewank2000algorithm,chen2016training} and model parallelism~\citep{shazeer2018mesh,huang2019gpipe,lepikhin2020gshard}.
Note that all three techniques are orthogonal and complementary to each other.
Next in section~\ref{sec:gradaccum}, we present an approach based on pipelining model parallelism and gradient accumulation.

\paragraph{Vanilla GradAccum.} Consider training a model weight vector $\theta$ to minimize a loss function $\mathcal{L}$. For a batch of $B$ examples $\{\rve_1, \rve_2, ..., \rve_B\}$, let $g_i$ be the gradient of $\mathcal{L}$ with respect to $\theta$ computed on example $\rve_i$, \ie,~$g_i = \nabla_{\theta} \mathcal{L}(\theta; \rve_i)$. In the standard minibatch setting, we update $\theta$ with the average batch gradient $\bar{g} = \big( \sum_{i=1}^{B} g_i \big) / B$. When our accelerator memory can only hold $M \ll B$ examples, GradAccum splits the batch of $B$ examples into smaller batches with at most $M$ examples, called \textit{microbatches}, then computes the gradients of the microbatches, and averages them.

We now analyze the steps of GradAccum. For simplicity, assume that $M$ evenly divides $B$, and that microbatch $i$-th consists of examples $\rve_j$'s with $(i-1)M+1 \leq j \leq iM$. With this assumption, the GradAccum procedure first initializes a zero vector $\bar{g}$ of the same size with $\theta$. Then, sequentially for each microbatch $i$-th, the microbatch gradient $c_i = \big( \sum\nolimits_{j=(i-1)M+1}^{iM} g_j \big) / M$ is added to $\bar{g}$. In the end, $\bar{g}$ holds the correct minibatch gradient, up to a normalization constant $K = B/M$.\looseness=-1

\paragraph{GradAccum cannot be naively applied to contrastive learning.} There are two properties that make GradAccum not applicable to contrastive learning.
First, in order to evaluate the loss $\text{ContrastiveLoss}_B$ in Equation~\ref{eqn:objective}, we need all entries of the similarity matrix $\mA$.
Hence, we cannot rely \textit{only} on examples in every microbatch $i$-th to compute the microbatch gradients $c_i$'s.
Second, GradAccum must allocate memory for the cumulative gradient $\bar{g}$.\footnote{It is worth noting that this is a common issue with GradAccum and is not specific to contrastive learning.}
As $\bar{g}$ has as many elements as $\theta$, its memory grows as we scale up the networks $F$ and $G$.
This growth becomes a more pronounced issue as we scale up our models.
For reference, our largest model has 3B weights, occupying roughly 11GB of accelerator memory. Spending another 11GB on $\bar{g}$, while possible, defeats the purpose of saving memory in GradAccum.
In the remaining of this subsection, we discuss how to modify GradAccum so that we can use it to scale up contrastive learning.\looseness=-1

\subsection{\label{sec:gradaccum}Modifying Pipelining and GradAccum for the Contrastive Loss}
\begin{table*}[h]
\begin{center}
  \resizebox{0.99\linewidth}{!}{ %
  \begin{tabular}{rll|c}
  \toprule
  
  \textbf{Inputs}
    & $\bullet$ Networks $F$, $G$ with a weight vector $\theta = [\theta_F, \theta_G]$, &
    & \multirow{5}{*}{\shortstack{\textbf{Memory}\\\textbf{Analysis}}} \\
    & $\bullet$ A minibatch of $B$ (image, text) pairs $\{(\vx_i, \vy_i)\}_{i=1}^{B}$,
    & $\triangleright$ $B$ is the contrastive batch size. & \\
    & $\bullet$ Microbatch size $M$. Assuming $M$ evenly devices $B$.
    & $\triangleright$ $M$ is the largest in-memory batch size. \\
  \textbf{Yields} & $\bullet$ Gradients $\nabla_{\theta} \text{ContrastiveLoss}$ for $B/M$
                  & $\triangleright$ The loss is computed as in Equation~\ref{eqn:objective} \\
                  & \textcolor{white}{$\bullet$} microbatches of the minibatch.
                  &
  \\
  \midrule
  1  & Allocate embedding matrices $\mX, \mY \in \sR^{D \times B}$
     & $\triangleright$ $D$ is the embedding size
     & $\Theta(BD)$ \\
  2  & \textbf{For} $i=1$ \textbf{to} $B/M$ \textbf{do}:
     & $\triangleright$ Sequentially compute the embeddings for \\
  3  & ~~~~Let $J \leftarrow \{j: (i-1)M+1 \leq j \leq iM \}$
     & ~~~microbatches of images and text sequences, \\
  4  & ~~~~$\mX_{:, J} \leftarrow F(\vx_J)$
     & ~~~\textit{not} saving the activations of $F$ and $G$.
     & $\Theta(M \cdot \text{Mem}(F))$ \\
  5  & ~~~~$\mY_{:, J} \leftarrow G(\vy_J)$ &
     & $\Theta(M \cdot \text{Mem}(G))$ \\
  6  & $\mA \leftarrow \big( \mX^\top \cdot \mY \big) / \tau$
     & $\triangleright$ $\mA \in \sR^{B \times B}$ is the similarity matrix \\
  7  & $\text{RowLoss}_B \leftarrow -\frac{1}{B} \sum_{i=1}^{B} \log{\frac{\mA_{i, j}}{\sum_{k=1}^{B} \mA_{i, k}}}$
     & $\triangleright$ The contrastive loss in Equation~\ref{eqn:objective}
     & $\Theta(B^2) $ \\
  8  & $\text{ColumnLoss}_B \leftarrow -\frac{1}{B} \sum_{j=1}^{B} \log{\frac{\mA_{i, j}}{\sum_{k=1}^{B} \mA_{k, j}}}$ & \\
  9  & $\text{ContrastiveLoss}_B \leftarrow \frac{\text{RowLoss}_B + \text{ColumnLoss}_B}{2}$ & \\
  10 & $d\mA \leftarrow \text{BackProp}(\text{ContrastiveLoss}, \mA)$
     & $\triangleright$ Back-prop to compute $\nabla_{\mA}$ \\
  11 & $d\mX \leftarrow \mY \cdot d\mA$
     & $\triangleright$ Because $\mA = \mX^\top \mY$ \\
  12 & $d\mY \leftarrow \mX \cdot d\mA$ & \\
  13 & \textbf{For} $i=1$ \textbf{to} $B/M$ \textbf{do}:
     & $\triangleright$ Repeat a forward pass on $F$, $G$ to \\
  14 & ~~~~Let $J \leftarrow \{j: (i-1)M+1 \leq j \leq iM \}$
     & ~~~back-prop the gradients from $d\mX$, $d\mY$ to \\
  15 & ~~~~$d\theta_F \leftarrow \text{ForwardAndBackProp}(d\mX_{:, J}, \theta_F)$
     & ~~~the weights $\theta_F, \theta_G$.
     & $\Theta(M \cdot \text{Mem}(F))$ \\
  16 & ~~~~$d\theta_G \leftarrow \text{ForwardAndBackProp}(d\mY_{:, J}, \theta_G)$ &
     & $\Theta(M \cdot \text{Mem}(G))$ \\
  17 & ~~~~\textbf{Yield} $d\theta = [d\theta_F, d\theta_G]$ & \\
  \bottomrule
  \end{tabular}
  } %
  \captionof{algorithm}{\label{alg:grad_accum}Pseudo code of our gradient accumulation process for the contrastive loss. Here $\text{Mem}(F)$, $\text{Mem}(G)$ denote the memory required for a pass for the networks $F$, $G$. As shown in our memory analysis, at the cost of repeating one forward pass for $F$, $G$ (lines 13-16), our procedure's peak memory footprint is dominated by $\Theta(M \cdot \max{\{ \text{Mem}(F), \text{Mem}(G) \}})$.}
\end{center}
\end{table*}

\paragraph{Microbatching the contrastive loss.}
To enable proper GradAccum, a key observation is that while we need the entire similarity matrix $\mA$ to compute $\text{ContrastiveLoss}_B$ in Equation~\ref{eqn:objective}, we do \textit{not} need to store all the intermediate results leading to the matrix in memory.
This observation immediately connects to re-materialization, which trades computation for memory by dropping some intermediate hidden states during the forward pass and re-computing them during back-propagation.
Following this insight, we propose to combine re-materialization with gradient accumulation by microbatching the contrastive loss and re-materializing each microbatch.

Specifically, we first run a forward pass on the networks $F$, $G$ to compute the entire similarity matrix $\mA$ while discarding all intermediate hidden states.
Then, we use $\mA$ to compute $\text{ContrastiveLoss}_B$ and the gradient $\nabla_{\mA} \text{ContrastiveLoss}_B$ and microbatch this gradient along the batch axis.
Finally, for each microbatch, we re-materialize the hidden states, \ie,~rerun the forward computation, and back-prop and accumulate the corresponding gradient microbatch of $\nabla_{\mA} \mathcal{L}_\text{c}$ into the weights of the networks $F$, $G$.

Algorithm~\ref{alg:grad_accum} presents this procedure in detail and provides the memory analysis for each step. As shown, our algorithm can compute the \textit{exact} microbatch gradients from an entire batch of $B$ examples, with the peak memory usage of $\Theta(M \cdot \max{\{ \text{Mem}(F), \text{Mem}(G) \}})$, instead of $\Theta(B \cdot (\text{Mem}(F) + \text{Mem}(G)))$.
We note that our algorithm can be flexibly modified to work different microbatch-sizes, \ie,~$M$, for the image network $F$ and the text network $G$. This flexibility allows for more efficient computations, \eg,~when one network is smaller than another and thus, can operate with larger microbatches.

\paragraph{Accumulating the microbatch gradients.}
Algorithm~\ref{alg:grad_accum} yields a stream of microbatch gradients $c_1, ..., c_{B/M}$, which need to be accumulated, \ie,~averaged,  into $\bar{g}$ to perform the batch weight update.
As discussed, we want to avoid allocating extra memory for $\bar{g}$.
To do this, we need two assumptions about our training implementation.
Our first assumption is that we use an optimizer which involves gradient moments~\citep{nesterov83a,tieleman12rmsprop,kingma15adam,loshchilov19decoupled,shazeer18adafactor}.
This assumption motivates our idea to avoid allocating $\bar{g}$: since the optimizer already allocates the memory for gradient moments, typically called \textit{slots}, we will directly accumulate the microbatch gradients $c_i$'s into these slots.\looseness=-1

We illustrate this idea with Adam~\citep{kingma15adam}, a popular optimizer that involves two gradient moments.
At training step $t$, Adam receives the averaged minibatch gradient $\bar{g}$ and makes the following updates to its gradient moments $v_1$ and $v_2$:
\begin{equation*}
\small
\begin{aligned}
  \bar{g} &= 1/B \cdot \sum\nolimits_{i=1}^{B} g_i
           = 1 / \underbrace{(B/M)}_{K} \cdot \sum\nolimits_{i=1}^{B/M} c_i \\
  v_1^{(t)} &= \beta_1 v_1^{(t-1)} + (1 - \beta_1) \bar{g} \\
  v_2^{(t)} &= \beta_2 v_2^{(t-1)} + (1 - \beta_2) \bar{g}^2
\end{aligned}
\end{equation*}
Accumulating the microbatch gradients $c_i$'s to $v_1$ is straightforward. We can simply modify $v_1$'s single update with $\bar{g}$ into $K = B/M$ updates as follows:
\begin{equation*}
\small
\begin{aligned}
  v_1 \leftarrow k_i v_1 + (1 - \beta_1) c_i,~\text{where}~k_i = \begin{cases}
  \beta_1 & \text{if $i = 1$} \\
  1 / K & \text{otherwise}
  \end{cases}
\end{aligned}
\end{equation*}
Unfortunately, the same approach is not applicable for $v_2$, as the square of the sum is generally different from the sum of the squares, \ie~$(\sum c_i)^2 \neq \sum c_i^2$. However, the difference between these two quantities turns out to be:
\begin{equation*}
\small
  \underbrace{\frac{1}{K} \sum_{i=1}^{K} c_i^2}_\text{sum of squares}
  - \underbrace{\Big( \frac{1}{K} \sum_{i=1}^{K} c_i \Big)^2}_\text{square of sum}
  = \expected{c_i^2} - \expected{c_i}^2
  = \variance{c_i},
\end{equation*}
which we can estimate. Indeed, since each $c_i$'s is the mean of $M$ per-example gradients $g_j$'s in the $i$-th microbatch, we can treat $c_i$'s as the population mean of $M$ observed examples drawn from a random variable $\rvg \sim \text{Uniform}\{g_1, ..., g_B\}$. This treatment allows us to use the familiar identity:
\begin{equation}
\small
\label{eqn:mean_variance}
\begin{aligned}
  \variance{c_i}
    = \variance{\frac{1}{M} \sum\nolimits_{j=(i-1)M + 1}^{iM} g_j}
    = \frac{\variance{\rvg}}{M}
\end{aligned}
\end{equation}
Therefore, to estimate $\variance{c_i}$, we only need to estimate $\variance{\rvg}$.
For this, we make the second assumption about our training: that we use a data parallelism setting with $R$ replicas. Under this assumption, each microbatch gradient $c_i$ is obtained from an all-reduce operation on $R$ replicas, each of which processes $M/R$ examples. Once again, treating these per-device gradients $d_1, ..., d_{R}$ as the population mean of $M/R$ observed examples for $\rvg$, we can apply Identity~\ref{eqn:mean_variance} to obtain: $\variance{d} = \variance{\rvg} / (M/R)$. This treatment allows us to perform GradAccum while avoiding to allocate $\bar{g}$.

\section{\label{sec:spmd_sharding}Batch Size Scaling with the Single-Program Multiple-Data (SPMD) Scheme}

In Section~\ref{sec:pipeline_and_accum}, we have seen that one circumvent the memory bottleneck of large models and large batch sizes for image-text contrastive learning by: (1) ``chunk'' a large batch of $B$ image-text pairs into arbitrarily smaller microbatches, (2) compute the gradient for each microbatch, and (3) accumulate them.
While such an approach is generic and can work for any global contrastive batch size $B$ and any microbatch size $M$, there are two steps in the approach that makes the resulting gradient \textit{inexact}.
The first inexact computation comes from the approximations when accumulating the microbatch gradients. This is a necessary tradeoff to avoid allocating the memory to accumulate the microbatch gradients.
The second inexact computation is more subtle, and is specific to our modeling choice, and has also been noted by~\citet{huang2019gpipe}.
Specifically, when our networks $F$ and $G$ depend on the batch,~\eg,~via the batch normalization layers in our image encoder $F$, then the outputs of the networks for multiple microbatches are generally different from those for one entire global batch.
When the microbatch size $M$ is too small compared to $B$, the discrepancies become very large and can cause the covariate shift in the image encoder $F$ which was the original motivation for batch normalization.
More recent image encoders can overcome such inexact computations because they avoid batch normalization by replacing it with layer normalization like Vision Transformer~\citep{dosovitskiy21vit}, or just do not use any normalization at all like NFNet~\citep{brock2021high}.
However, the inexact gradient accumulation remains even for such models.

To overcome these inexact computations, in this section, we discuss an alternate approach to circumvent the memory bottleneck.
This approach is based on the SPMD programming scheme.
We find that not only does our SPMD method provide exact computations which lead to better results than pipelining and GradAccum, but our SPMD method also has a better latency per training step. 
However, as we shall see in Section~\ref{sec:model_parallelism} and Section~\ref{sec:rematerialization}, our SPMD method requires several manual designs, which make it less generic than pipelining and GradAccum.

\subsection{\label{sec:model_parallelism}Weight Sharding}
As model sizes grow, model weights occupy a significant part of accelerator memory.
In modern optimizers for deep learning models, such as Adam~\citep{kingma15adam}, RMSprop~\citep{tieleman12rmsprop}, and AdamW~\citep{loshchilov19decoupled}, every weight tensor is additionally accompanied by the first and second gradient moments, hence tripling its memory footprint.
Furthermore, in the vanilla data parallelism training, all these weights are replicated to \textit{all} accelerators.
In our experiments with a relatively large model size, roughly 4GB of accelerator memory is occupied by these weights and their gradient moments, which is significant for the typical 16GB memory in an accelerator in 2022, such as a Google TPU core or an Nvidia RTX 3080 GPU.

Here, we split the weight tensors in our encoder networks,~\ie~$F$ and $G$ in Section~\ref{sec:background}, into multiple accelerator cores, and only combine these tensors together when the whole tensor is needed to perform certain computations.
Note that upon splitting a weight tensor to multiple cores, we also split its first and second gradient moments in the similar way.
Figure~\ref{fig:sharding} illustrates our weight sharding strategy on the 2D convolution operation which is prevalent in image encoder models.

\begin{figure}[h!]
\centering
\includegraphics[width=0.5\linewidth]{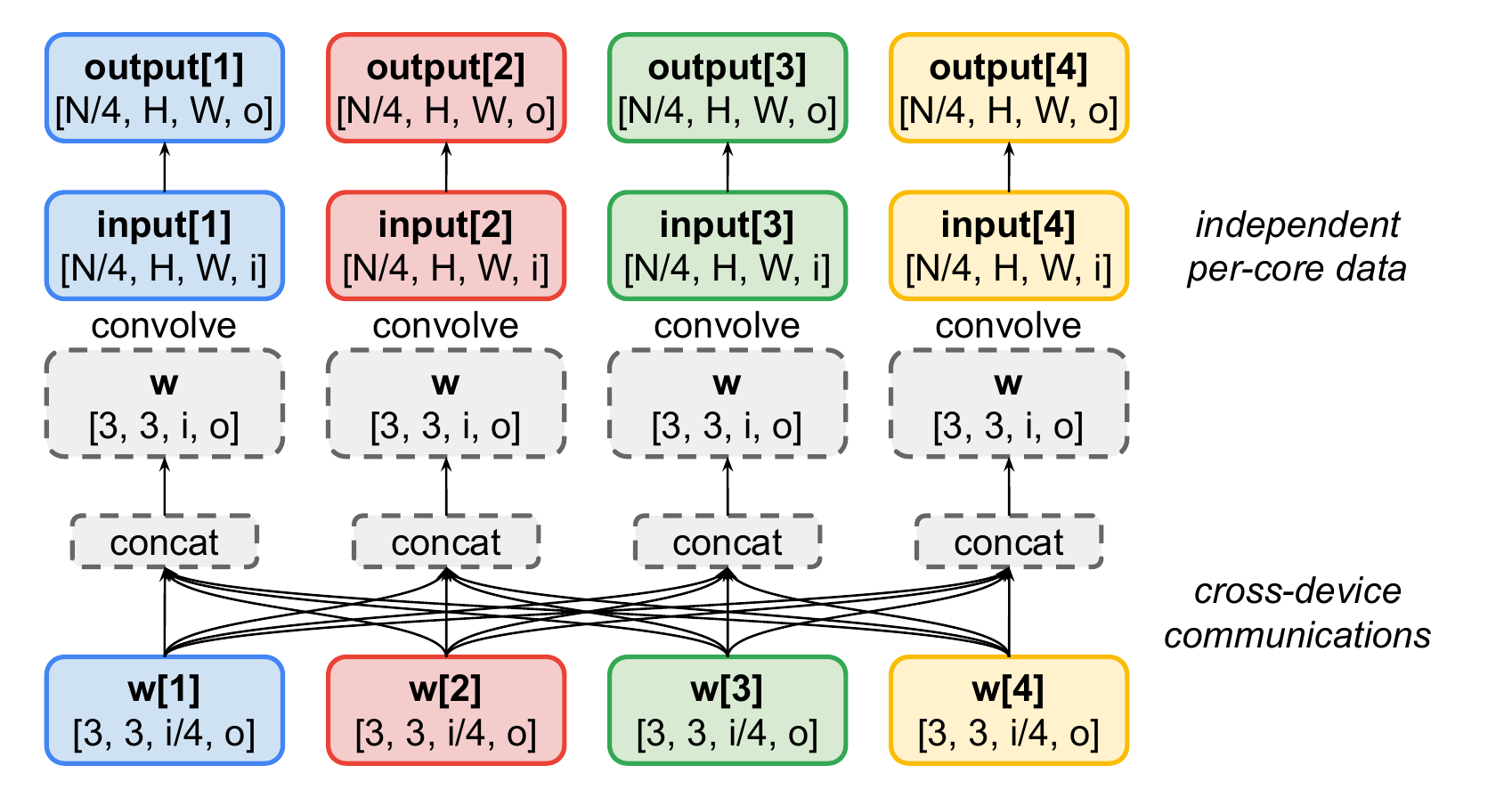}
\caption{\label{fig:sharding}An illustrative example for our model parallelism design. Shown is 2D convolution operation with a 3x3 kernel sharded to 4 cores. Gray cells represent the tensor values that are mirrored across all cores, while cells of other colors represent per-core independent tensor values. The convolution's input is a tensor of shape $[N, H, W, i]$ which is sharded along its first dimension so that each core has processes a tensor of shape $[N/4, H, W, i]$. The convolution's kernel is a tensor of shape $[3, 3, i, o]$, but each core only stores one shard of the kernel which has size $[3, 3, i/4, o]$. Before convolving, every core receives the kernel shards from all other cores and concatenate the shares, forming the complete kernel of size $[3, 3, i, o]$. After convolving, the complete kernel is discarded from all cores' memory.}
\end{figure}

Our approach is based on the Single-Program Multiple-Data (SPMD) technique, which has been successfully applied to train large language models in previous works such as in~\citet{xu2021gspmd,lepikhin2020gshard}.
In the SPMD technique, we define a computational graph which represents our entire training program.
This computational graph is compiled once, and then is replicated identically to all computational cores to run the training program.
While all of our computational cores run an identical program, they are allowed to receive different inputs and hence can produce different outputs.
These inputs and outputs can be organized in certain ways to define arbitrarily complex model parallelism strategies.
Next, we describe how we apply the SPMD technique on our \textit{model weights} only.

Our training program runs typically on a cluster of 2048 TPUv3 cores. We partition these 2048 cores into $R$ \textit{replicas}, each of which uses $2048/R$ cores.
The value of $R$ governs how the weights of our image encoder $F$ and our text encoder $G$ are stored in the memory of our 2048 cores.
In particular, all weight tensors in the networks $F$ and $G$ are split into $R$ equal parts, each lives in one of the $R$ cores in a replica.
Note that since we have $2048 / R$ replicas, the weights of our image and text encoders are still replicated for $2048 / R$ times.
For instance, our cores $1^\text{st}$, $2^\text{nd}$, ... $R^\text{th}$ can each store $1/R$ of the weight tensors, and then the cores $R+1^\text{st}$, $R+2^\text{nd}$, ..., $2R^\text{th}$ store an identical copy of these tensors.
Thus, using fewer replicas and more cores per replica leads to a better memory utilization, at a higher overhead for cross-cores communications.
We empirically find that using 512 replicas and 4 cores per replica offers a good balance.

It is important to note that we only apply SPMD on our model weights, and not on any other steps of our computations.
This means that if our training program receives an input batch of $B$ examples, then these $B$ examples are distributed equally to all our 2048 cores.
In other words, each of our 2048 cores processes $B / 2048$ examples, regardless of the value of $R$.
We find that this design choice disentangles our weight sharding strategy from the rematerialization strategy, as described next in Section~\ref{sec:rematerialization}.

\subsection{\label{sec:rematerialization}Rematerialization}
\begin{figure*}[tb!]
\centering
\includegraphics[height=0.98\textwidth,angle=270]{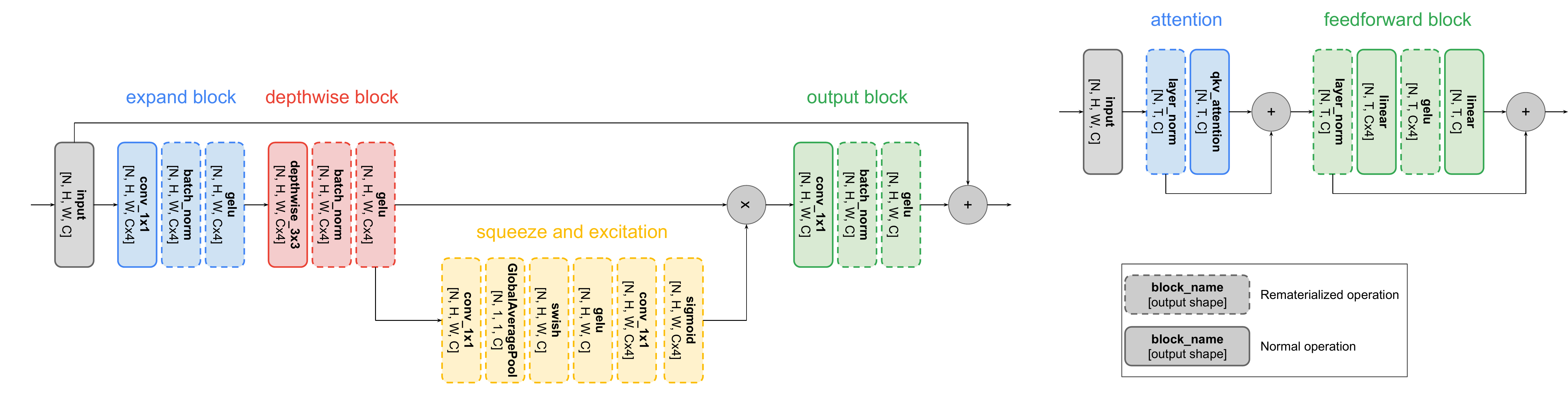}
\caption{\label{fig:remat}Generic rematerialization map for the blocks in our CoAtNet models. \textbf{Left:} in the Mobile Inverse Convolution blocks (MBConv), all batch normalization layers and activation layers, as well as all layers in the squeeze and excitation steps, are rematerialized. \textbf{Right:} in Transformer blocks, only the layer normalization layers and activation layers are rematerialized.}
\end{figure*}

The technique of rematerialization, also widely known as gradient checkpointing~\citep{chen2016training}, preserves the accelerator memory while training neural networks.
It works by \textit{not} saving certain values from a network's forward pass, and recompute them in the backward pass \textit{only} when their values are needed for a particular calculation.
For instance, if our image encoder $F$, as discussed in Section~\ref{sec:background} has 100 layers, a rematerialization program can decide that after the forward pass, only the values of layers $10^\text{th}$, $20^\text{th}$, ..., $90^\text{th}$ are kept in an accelerator's memory, while the values of other layers are removed.
If all layers of $F$ consumes similar memory, this rematerialization program has reduced the memory cost by 10 times, at the trade off that the values of the unsaved layers in the forward pass, such as layer $21^\text{st}$ or layer $72^\text{nd}$, have to be recomputed in the backward pass.

We select which layer to rematerialize in our image encoder $F$ and our text encoder $G$ based on a simple heuristic.
Ideally, we want to rematerialize the layers that are fast to recompute but consumes the more memory.
Since we utilize weight sharding, as described in Section~\ref{sec:model_parallelism}, the computations that involve weights are slower than normal because of their overhead time for cross-core communications.
As such, we keep almost all layers that involve weights, such as convolution, attention, and dense feed-forwards, in our accelerator's memory. In contrast, layers that do not involve weights, such as activation functions, batch normalization, and layer normalization, are all rematerialized. Figure~\ref{fig:remat} illustrates our rematerialization strategy for all three block types in our image and text encoders: the mobile-inverse convolutional block~\citep{tan19efficientnet,dai21coatnet}, the attention block, and the feed-forward blocks~\citep{vaswani2017attention}.

We find this design choice beneficial, because in modern encoder architectures, every layer that involves weights is typically followed by a normalization layer or an activation layer.
As such, our design allows more than half of the encoder's activation values to be removed from the accelerator's memory after each forward pass, while leaving only the light computational steps to be repeated in each backward pass.
We empirically find that weight sharding and rematerialization, each of our forward-backward pass is 1.4 times slower than the vanilla implementation of the same batch size.

\paragraph{Exceptions to our heuristics.}
Certain parts of our encoders do not follow these general heuristics, as we find that doing so saves a certain amount of time at the cost of using a little extra memory.
Here, we describe these exceptions:
\begin{enumerate}
  \item All weights in batch normalization and layer normalization in our models, including the $\beta$'s and $\gamma$'s and the moving average statistics of batch normalization, are \textit{not} sharded. Instead, these weights are replicated to all computational cores to avoid cross-cores communications, because they are one dimensional vectors which do not occupy much memory.
  \item All computations the Squeeze-and-Excitation blocks (SE;~\citep{hu2018squeeze}) of our models are rematerialized, including the convolutions. This is because these SE blocks only involve 1x1 convolution with reduced internal channels, making them less costly to recompute. In addition, all the weights of these 1x1 convolutions are replicated to all of our cores, because they have a small memory footprint but are reused in the backward pass for rematerialization.
\end{enumerate}

\subsection{\label{sec:runtime_compar}Comparison with Pipelining and Gradient Accumulation}
\begin{table}[htb!]
\centering
\resizebox{0.65\linewidth}{!}{ %
\begin{tabular}{cccccccccccccc}
\toprule
  \textbf{Model} &
  \multirow{2}{*}{\textbf{Methods}} &
  \textbf{Batch} & 
  \multicolumn{3}{c}{\textbf{Step time} (millisecs)} & \textbf{Memory} \\
  \cmidrule(lr){4-6}
  \textbf{Size} & & ($B$) &
  FWD & BWD & Total & (GB) \\
\midrule
  \multirow{7}{*}{Small}
    & Data parallelism      & $2^{16}$ & 28.6  & 81.2   & 128.5  & 4.2  \\
    & Data parallelism      & $2^{17}$ & 65.8  & 144.5  & 230.3  & 6.4  \\
    & Data parallelism      & $2^{18}$ & 131.2 & 282.9  & 428.0  & 9.9  \\
    & Data parallelism      & $2^{19}$ & 261.0 & 559.4  & 842.4  & 14.8 \\
    & Data parallelism      & $2^{20}$ & $-$   & $-$    & $-$    & OOM  \\   
    \cmidrule(lr){2-7}
    & Pipeline \& GradAccum & $2^{20}$ & 301.7 & 1497.4 & 1983.1 & 11.2 \\   
    & SPMD                  & $2^{20}$ & 271.8 & 1311.8 & \textbf{1631.6} & 15.2 \\   
\midrule
  \multirow{9}{*}{Medium}
    & Data parallelism & $2^{16}$ & 129.5  & 270.8  & 437.6  & 10.7 \\
    & Data parallelism & $2^{17}$ & 237.7  & 536.4  & 806.1  & 13.8 \\
    & Data parallelism & $2^{18}$ & $-$    & $-$    & $-$    & OOM  \\
    \cmidrule(lr){2-7}
    & Pipeline \& GradAccum & $2^{18}$ & 597.7  & 1677.1 & 2677.3 & 12.6 \\
    & SPMD                  & $2^{18}$ & 599.3  & 1407.2 & \textbf{2018.6} & 12.1 \\
    \cmidrule(lr){2-7}
    & Pipeline \& GradAccum & $2^{19}$ & 1393.1 & 4087.8 & 5601.2 & 12.8 \\
    & SPMD                  & $2^{19}$ & 1402.7 & 3141.1 & \textbf{4912.3} & 14.1 \\
    \cmidrule(lr){2-7}
    & Pipeline \& GradAccum & $2^{20}$ & 3014.3 & 6129.9 & 9781.4 & 12.6 \\
    & SPMD                  & $2^{20}$ & 2909.7 & 4912.3 & \textbf{8361.1} & 15.4 \\   
\bottomrule
\end{tabular}
} %
\caption{\label{tab:speed_and_memory}Comparison between our SPMD programs and our Pipelining \& GradAccum programs. Shown are the step times and memory footprints of our small-sized and medium-sized models at different batch sizes. With the same model size and batch size, our SPMD design results in a larger device memory footprint and Pipelining \& GradAccum, but SPMD is faster.}
\end{table}

We measure and compare the step time and peak memory usage of our SPMD approach and our Pipelining and GradAccum approach from Section~\ref{sec:pipeline_and_accum}.
We measure these pieces of information as the contrastive batch size $B$ becomes larger.
In particular, we start with $B=2^{16}$ and double $B$ until we reach $B=2^{20}$.
When $B$ increases, but our model can still fit into the our memory only with vanilla data parallelism, we measure the step time and the peak memory using this setting.
These measures serve as the reference to quantify how much overhead is introduced by pipelining, SPMD, or rematerialization as a whole.
When $B$ grows and our models can no longer train with merely data parallelism, we experiment with the Pipelining and GradAccum, and compare the resulting programs with the SPMD programs of the same model and batch size.
Note that for the pipelining approach, we set the microbatch size to the largest size that vanilla data parallelism can fit in our accelerator's memory.

For all settings, we profile a model in 15 seconds.
Our profiling tool tracks the time and memory usage of our models during their forward and backward passes in each step.
Note that our backward pass time \textit{includes} the time our models spend on their rematerialized computations (as discussed in Section~\ref{sec:rematerialization}), as these computations happen while the models compute their gradients.
Note that other than these forward and backward passes, every step of our models has some extra overheads for miscellaneous computations that are not recorded by our profiling tool,~\eg,~gradient clippings and updating model parameters.

All of our measurements are reported in Table~\ref{tab:speed_and_memory}.
From the table, it can be seen that our model parallelism strategy leads to a faster overall step time, compared to the pipelining approach in the same setting.
Breaking down these step times, it can be seen that the run time of our strategy's model forward pass is very close to the run time of pipelining's forward pass, but our backward time is often a lot faster.
For instance, in our largest setting, with the medium-sized model and the contrastive batch size $B=2^{20}$, our backward time is more than 1.2 seconds faster than that of the pipelining approach, amounting to about 10\% of the total step time.
Additionally, our strategy also has a faster total step time, perhaps because do not need to spend extra time to accumulate the microbatch gradients like the pipelining approach.

Finally, we note that as $B$ grows larger, the SPMD approach typically occupies more accelerator memory than does the pipelining approach.
This is because in the pipelining approach, increasing the contrastive batch size $B$ only leads to more microbatches, but does not change the micro batch size, and so the accelerator's memory remains constant.
As such, the pipelining approach is still applicable if $B$ grows larger than $2^{20}$, but the SPMD strategy has to be redesigned,~\eg~by deciding to rematerialize a larger portion of our image and text encoder.

\section{\label{sec:theory}Theoretical Insight on the Role of Contrastive Batch Size} 
In this section, we provide a theoretical insight that increasing the contrastive batch size tends to improve the performance of the final model, which  motivates and partially justifies the design of our new algorithm in the previous section. Let $\hy_{1},\dots,\hy_{B}$ be the sequence of the text sentence inputs used in training. Similarly,  let  $\by_{1},\dots,\by_{M}$ be the sequence of the text sentence inputs used in testing. We then define the
 normalized training  loss by
 \begin{align*}
 \hell_B(x,y) =-  \frac{B\exp(F(x)\T G(y))}{\sum_{k=1}^B \exp(F(x)\T G(\hy_k))}
 =-  \frac{\exp(F(x)\T G(y))}{\frac{1}{B}\sum_{k=1}^B \exp(F(x)\T G(\hy_k))},
 \end{align*}
where we multiply $B$ to scale the loss correctly in the regime of $B\rightarrow \infty$. That is, since the unnormalized loss goes to zero ($\hell_B(x,y)/B \rightarrow 0$) as the batch size approach infinity ($B\rightarrow \infty$), analyzing the unnormalized version of the loss $\hell_B(x,y)/B$ can mistakenly predict benefits of the large contrastive batch size. We avoid this with the normalization.  Similarly, we define the normalized testing loss by
$$
\bell_M(x,y)= -  \frac{\exp(F(x)\T G(y))}{\EE_{\by}[ \exp(F(x)\T G(\by))]}.
$$
Then, the prediction at testing time for a new input $x$ is given by 
\begin{align*}
\text{pred}(x)=\argmax_{j \in \{1,2,\dots,M\}} F( x)\T G(\by_{j})
 =\argmin_{j \in \{1,2,\dots,M\}} \bell_M(x, \by_{j}).
\end{align*}
Therefore, $\hell_B(\hx_{i}, \hy_{i})$ is minimized during training for training points $(\hx_{i}, \hy_i)$ while we want to minimize $\bell_M(x, y)$ to make a prediction at a new point $x$. This leads to the question of the generalization from training to unseen-data, which can be studied by analyzing the upper bound on the following quantity: 
$$
\EE_{x,y}[\bell_M(x, y)] -\hEE_S[\hell_B(x, y)],
$$   
where $\hEE_S[\hell_B(\hx, \hy)]$ is the empirical training loss with a dataset, $S=((\hx_i, \hy_i))_{i=1}^m$, of size $m$.
To analyze this in a statistical setting, we define a vector $v \in \RR^D$ by $
v_{i}=F(x)_{i} - \frac{\sum_{k=1}^B \exp(F(x)\T G(\hy_k))G(\hy_k)_i}{\sum_{k=1}^B \exp(F(x)\T G(\hy_k))}
$ for all $i\in \{1,\dots,D\}$, and assume that  $\hy_{1},\hy_{2},\dots,\hy_{B}\stackrel{iid}{\sim}p_{y}$,   $\by_{1},\by_{2},\dots,\by_{M}\stackrel{iid}{\sim}p_{y}$,   $\exp(F(x)\T G(y))\le c_1$, $\hell_B(x,y) \le c_2$,  $\|F(x)\|_2\le c_3$, $
\frac{\exp(F(x)\T G(y))}{\frac{1}{B}\sum_{k=1}^B \exp(F(x)\T G(\hy_k))} \le c_4$, $\|F(x_{i})\|_{2} \le c_5$, $\|v\|_2 \le c_6$,   $\|y\|_{2}\le c_7$,   and $\|x\|_{2}\le c_8$, with probability one. 
  Moreover, by defining  $\gamma(x)=\EE_{\by } [\exp(F(x)\T G(y))]- \frac{1}{B}\sum_{k=1}^B \exp(F(x)\T \allowbreak G(\hy_k))$, we assume that $\gamma$ is  $c_9$-Lipschitz; i.e., $|\gamma(x)-\gamma(x')|\le c_9 \|x-x'\|_2$ for all $x\in \Xcal \subseteq \RR^\kappa$.
 To provide a concrete insight, we consider standard deep neural networks, of the form
\begin{align*}
&G(y)= (\omega_{L} \circ \sigma_{L-1}\circ \omega_{L-1}\circ \sigma_{L-2}\cdots \sigma_1\circ \omega_1)(y),
\\ & F(x)= (\omega_{L'} '\circ \sigma_{L'-1} '\circ \omega_{L'-1} '\circ \sigma_{L'-2}' \cdots \sigma_1 '\circ \omega_1')(x),
\end{align*}
where $\omega_l(q)=W_{l}q$ represents the linear transformation and $\sigma_{l}$ is an element-wise nonlinear activation function. Similarly,  $\omega_l'(q)=W_{l}'q$ and $\sigma_{l}'$ is an element-wise activation function.

The following theorem provides an  insight on the role of the contrastive batch size to close the accuracy gap from contrastive models to their supervised counterparts:

\begin{restatable}{theorem}{thma} \label{thm:1}
Suppose that the activation functions $\sigma$ and $\sigma_{l}'$ are 1-Lipschitz and positive homogeneous for all $l \in [L-1]$. Let  $ \Gcal=\{y  \mapsto G(y): (\forall l \in[L-1])[\|W_{l}\|_F \le M_l] \wedge  \|(W_{L})_{k}\|_F \le M_{L,k}\}$ and  $ \Fcal=\{x  \mapsto F(x): (\forall l \in[L'-1])[\|W_{l}'\|_F \le M_l'] \wedge  \|(W_{L'}')_{k}\|_F \le M_{L',k}'\}$ where $(W_{l})_{k}$ is the $k$-th row  of $W_{l}$. Then, for any $\delta>0$, with probability at least $1-\delta$, the following holds for all $F \in \Fcal$ and $G \in \Gcal$:
\begin{align*}
\EE_{x,y}[\bell_M(x, y)] -\hEE_S[\hell_B(x, y)] 
\le \frac{Q_1}{\sqrt{m}}  +\frac{Q_{2 }}{\sqrt{2B}}+c_2 \sqrt{\frac{\ln(2/\delta)}{2m}}, 
\end{align*}
where $Q_1 =  2 \sqrt{2} c_4 \sqrt{c_5^{2}+c_6^{2}} (\tilde Q_{1,1} +  \tilde Q_{1,2})$,  $Q_{2} = c_1\EE_{x,y}[A(x,y)](  \tilde Q_{2,1} +  \tilde Q_{2,2})$,  
\begin{align*}
&  \tilde Q_{1,1} =c_7 (\sqrt{2 \log(2) L }+1)\left(\prod_{l=1}^{L-1} M_l\right)\left(\sum_{k=1}^D M_{L,k}\right),
\\ & \tilde Q_{1,2} =c_8 (\sqrt{2 \log(2)L' }+1)\left(\prod_{l=1}^{L'-1} M_l'\right)\left(\sum_{k=1}^D M'_{L',k}\right),
\\ & \tilde Q_{2,1} =2\sqrt{2}c_8 c_9+c_1 \sqrt{\kappa\ln( \sqrt{\kappa B}/\delta)}, 
\\&\tilde Q_{2,2} =2\sqrt{2}c_3c_7 (\sqrt{2 \log(2) L }+1)\left(\prod_{l=1}^{L-1} M_l\right)\sqrt{\sum_{k=1}^D M_{L,k}^2}, 
\end{align*}
and $A(x,y)=\frac{\exp(F(x)\T G(y))}{\left(\frac{1}{B}\sum_{k=1}^B \exp(F(x)\T G(\hy_k))\right)\EE_{\by}[ \exp(F(x)\T G(\by))]}$.
\end{restatable}
\begin{proof} \ The proof is presented in Appendix \ref{app:proof}. 
\end{proof}
Theorem \ref{thm:1} shows that the generalization gap $\EE_{x,y}[\bell_M(x, y)] -\hEE_S[\hell_B(x, y)]$ approaches zero as the number of samples $m$ and the contrastive batch size $B$ increase, at the rate of $O(\frac{1}{\sqrt{m}}+\frac{1}{\sqrt{B}} )$, with high probability. This  shows the importance of the contrastive batch size $B$ with the week supervision setting without explicit labels: i.e., if $B$ is small, the generalization gap can be large, even with a lot of training samples with large $m$. This is different from a classification task with a standard supervision, where a large $m$ is sufficient to reduce  the generalization gap. In sum, Theorem \ref{thm:1} provides  the insight that we should increase both the number of samples $m$ and the contrastive batch size $B$ to improve the performance of the final model via the weekly supervised contrastive learning.

For general models beyond the standard deep neural networks, the following theorem provides a similar insight on the importance of the contrastive batch size:  
   
\begin{restatable}{theorem}{thmb} \label{thm:2}
Let $\Fcal$  be a  set of maps $x\mapsto F(x)$ and $\Gcal$ be a set of maps $y\mapsto G(y)$. Then, for any $\delta>0$, with probability at least $1-\delta$, the following holds for all $F \in \Fcal$ and $G \in \Gcal$:
\begin{align*}
\EE_{x,y}[\bell_M(x, y)] -\hEE_S[\hell_B(x, y)]
 \le   \frac{ C_1 }{\sqrt{2B}}  +c_2 \sqrt{\frac{\ln(2/\delta)}{2m}}
 +C_{2} \sum_{k=1}^{D}\left(\Rcal_ m( \Fcal_{k})+\Rcal_ m( \Gcal_{k})   \right) +C_{3}\tilde \Rcal_B(G) ,
\end{align*}
where   $\Rcal_{m}(\Hcal):=\EE_{S,\xi}[\sup_{h \in\Hcal}\frac{1}{m} \sum_{i=1}^m \xi_i h(x_{i},y_{i})]$,
   $\Fcal_k=\{x \mapsto F(x)_k : F \in \Fcal\}$,  $\Gcal_k =\{y \mapsto G(y)_{k} : G \in \Gcal\}$,
 $\tilde \Rcal_B(G)=\EE_{y,\xi}\left[\sup_{ G \in \Gcal} \frac{1}{B}  \left\|\sum_{i=1}^B \xi _{i}  G(y_{i})\right\|_2\right]$,  $C_1=c_1\EE_{x,y}[A(x,y)]\tilde Q_{2,1}$,  $C_{2}=2\sqrt{2} c_4 \sqrt{c_5^{2}+c_6^{2}}$, and $C_3=2c_1c_3\EE_{x,y}[A(x,y)]$. Here,  $\xi_1,\dots,\xi_m$ are independent uniform random variables taking values in $\{-1,1\}$. 
\end{restatable}
\begin{proof} \ The proof is completed in Appendix \ref{app:proof}. 
\end{proof}
Similarly to Theorem \ref{thm:1}, Theorem \ref{thm:2} shows that the  gap  $\EE_{x,y}[\bell_M(x, y)] -\hEE_S[\hell_B(x, y)]$ decreases as the number of samples $m$ and the contrastive batch size $B$ increase, at the rate of $O(\frac{1}{\sqrt{m}}+\frac{1}{\sqrt{B}} )$, with high probability, for general models $F$ and $G$. Because we are not specifying the types of the models  $F$ and $G$, we incur the additional terms that capture the model complexity of  $F$ and $G$: i.e.,  $\Rcal_ m( \Fcal_{k})+\Rcal_ m( \Gcal_{k})$ and $\tilde \Rcal_B(G)$. The values of these model complexity terms differ for different models and   typically scale as $\Rcal_ m(\Fcal_{k})+\Rcal_ m( \Gcal_{k}) = O(\frac{1}{\sqrt{m}})$ and $\tilde \Rcal_B(G)=O(\frac{1}{\sqrt{B}})$ in terms of $m$ and $B$ as illustrated in the proof of Theorem \ref{thm:1} for standard deep neural networks. 
Therefore, it is desirable to use a large contrastive batch size, which motivates our new algorithm in the previous section.

\section{Data and Model Scaling} 
\subsection{Larger image-text dataset} 
Starting from the ALIGN dataset, which contains 1.7B weakly-aligned image-text pairs~\citep{jia21scaling}, we collect 5B more image-text pairs, hence expanding the dataset size by roughly 4 times.
We acquire these 5B image-text pairs from the JFT dataset. In the JFT dataset, each image is associated with one or multiple classes. We convert these classes into a text sequence: \textit{``\{class\_1\} and \{class\_2\} and ... and \{class\_k\}''}.
We combine the instances from JFT into ALIGN, forming our extended dataset, which we denote by ALIGN+JFT.

To tokenize the texts from ALIGN+JFT, we randomly sample 200M sentences and use them to train a sentence piece model~\citep{kudo18sentencepiece} with a vocabulary size of 32K pieces.
Using this tokenizer, we filter and discard the text sequences which are longer than 64 tokens.
In our preliminary experiments, we find that using a tokenizer directly learned from ALIGN+JFT and adapting this filtering step can boost our top-1 accuracy on ImageNet ILSVRC-2012 by more than 1\%.\looseness=-1

\subsection{Larger Model Architectures} 
We find that for the same computational budget, it is more beneficial to invest in scaling up the image encoder, rather than the text encoder.
Thus, for our image encoder, we use the largest CoatNet architecture~\citep{dai21coatnet} due to its proven large learning capacity. This network has convolution layers followed by attention layers.
For our text encoder, we use a simple transformer~\citep{vaswani2017attention}.
Unlike ALIGN~\citep{jia21scaling} which extracts the final text representations using a [CLS] token similar to BERT~\citep{devlin18bert}, we average the representations across all steps at the top layer of our transformer.

By experimenting with the scaling benefits for small models and generalizing these findings to larger models, we choose three model sizes, termed BASIC-\{S,M,L\} for Small, Medium, and Large. In Appendix~\ref{sec:model_size}, we report our architectures and their computational costs and provide a small-scale study on the effects of scaling model sizes.
% \hieu{include results about computational budget spending}.

\section{\label{sec:pretraining}Pretraining and Finetuning}
To further speed up the training of our networks, we make use of pretraining. In our experiments, we first pretrain the image encoder on a large labeled dataset using the standard softmax classification loss.
After pretraining the image encoder, we fix all of its weights and just train the text encoder using contrastive learning.
Compared to contrastive learning with GradAccum, the pretraining-finetuning procedure is much more efficient in terms of peak memory usage.
This is because we never have to compute the gradients of \textit{both} the image encoder and the text encoder, which allows automated compiler optimizations to free up unused memory on-the-fly.

Despite its reduced memory usage, we find that this pretraining-finetuning scheme has a weakness: it never exposes the image encoder to noisy image-text data, which makes the image encoder fail on certain tasks.
For instance, while some pretrained-and-finetuned models achieve similar accuracy to their contrastive counterparts on ImageNet or CIFAR, they completely fail on an easier task -- MNIST.
This is because our pretraining labeled dataset, which mostly consists of natural images, has very few digit images.
Meanwhile, our noisy image-text dataset has plenty instances that can teach a model certain optical character recognition skills.\looseness=-1

As will be shown in Section~\ref{sec:exp}, our best experimental results are achieved using a hybrid procedure.
First, we pretrain the image encoder on a large labeled dataset, then fix its weights and train the text encoder using the contrastive loss on our image-text dataset.
Finally, we finetune both image and text encoders, using our GradAccum technique when needed.
In Section~\ref{sec:ablation}, we present ablation studies to analyze the effects of pretraining, finetuning, and other alternative training procedures.\looseness=-1

\section{\label{sec:exp}Experiments}

\subsection{\label{sec:training}Training details}
%\paragraph{Training recipe.} The best recipe to train our models has three phases. The first phase pretrains our image encoder on a large labeled dataset using the classification softmax cross-entropy loss. The second phase keeps the image encoder's pretrained weights fixed and trains the text encoder using the contrastive loss on ALIGNv2. Finally, the third phase trains both networks with a smaller learning rate to adapt the image encoder to ALIGNv2. We find this recipe by selecting the model with the highest ImageNet validation accuracy for our smallest model, BASIC-S, and reuse the recipe for larger models BASIC-M and BASIC-L. In Section~\ref{sec:image_classification}, we report the performances of our models after this third phase. In Section~\ref{sec:ablation}, we will present ablation studies to understand the contributions of each phase and of our datasets.\looseness=-1

\paragraph{Labeled data for pretraining.} For pretraining (Section~\ref{sec:pretraining}), we use the JFT dataset.
This dataset has been used in previous publications~\citep{zhai21scaling,dosovitskiy21vit,kolesnikov20bit}, but it has been constantly expanded. The JFT version used in our experiments has 5B images, each of which can be associated to one or multiple labels out of 29K possible classes.\looseness=-1

\paragraph{Data filtering.} A problem with training on large auto-curated datasets like ALIGN and JFT is that these datasets might unintentionally contain examples from our test sets.
% While CLIP~\citep{radford21learning} has reported mild changes in evaluation results due to these overlapped examples, such contaminations could still contribute a confounding factor when we draw conclusions about our models' performance.
To avoid such contaminations, we filter all instances in our training data that has a structural similarity index (SSIM~\citep{wang2004image}) of at least 0.5 with any image from our evaluation benchmarks.

\paragraph{Optimizer.} We train our models with our own optimizer called AdaFactorW, adapted from two existing ones: AdaFactor~\citep{shazeer18adafactor} and AdamW~\citep{loshchilov19decoupled}.
Specifically, we factorize our second gradient moments like AdaFactor, and decouple the weight decay from all moments like AdamW.
To further save memory, we follow~\citet{zhai21scaling} and store the first gradient moments in \texttt{bfloat16}. We observe, however, that while we can \textit{store} these moments in \texttt{bfloat16}, we need to convert them into \texttt{float32} prior to computing our weight updates to avoid numerical instability.
% \hieu{make an algorithm in appendix.}

\begin{table*}[htb!]
\centering
\resizebox{0.98\textwidth}{!}{ %
\begin{tabular}{r|ccccccccccccccccc}
  \toprule
  \rotatebox[origin=l]{90}{\textbf{Datasets}} &
  \rotatebox[origin=l]{90}{Birdsnap} &
  \rotatebox[origin=l]{90}{Caltech101} &
  \rotatebox[origin=l]{90}{CIFAR10} &
  \rotatebox[origin=l]{90}{CIFAR100} &
  \rotatebox[origin=l]{90}{DTD} &
  \rotatebox[origin=l]{90}{EuroSAT} &
  \rotatebox[origin=l]{90}{Flowers} &
  \rotatebox[origin=l]{90}{Food101} &
  \rotatebox[origin=l]{90}{ImageNet} &
  \rotatebox[origin=l]{90}{MNIST} &
  \rotatebox[origin=l]{90}{IIIT-Pets} &
  \rotatebox[origin=l]{90}{PCam} &
  \rotatebox[origin=l]{90}{RESISC45} &
  \rotatebox[origin=l]{90}{STL10} &
  \rotatebox[origin=l]{90}{SUN397} &
  \rotatebox[origin=l]{90}{UCF101} &
  \rotatebox[origin=l]{90}{VOC2007}
  % \rotatebox[origin=l]{90}{\textbf{Average}} &
  % \rotatebox[origin=l]{90}{\makecell{\textbf{Trimmed}\\\textbf{Average}}}
  \\
  \midrule
  ResNet-50 & 32.6 & 82.1 & 75.6 & 41.6 & 41.7 & 41.1 & 65.9 & 81.1 & 59.6 & 66.6 & 85.4 & 57.6 & 54.2 & 94.3 & 59.6 & 63.6 & 82.1 \\ % & 63.8 & 63.9 \\
  \multirow{2}{*}{BASIC-S} & 38.6 & 91.6 & 86.4 & 57.8 & 54.3 & 29.1 & 76.8 & 86.0 & 71.9 & 32.4 & 93.2 & 54.3 & 53.5 & 96.7 & 67.3 & 65.5 & 83.4 \\ % & \textbf{66.5} & \textbf{67.0} \\
  & \gain{6.0} & \gain{9.5} & \gain{10.8} & \gain{16.2} & \gain{12.6} & \loss{-12.0} & \gain{10.9} & \gain{4.9} & \gain{12.3} & \loss{-34.2} & \gain{7.8} & \loss{-3.3} & \loss{-0.7} & \gain{2.4} & \gain{7.7} & \gain{1.9} & \gain{1.3} \\
  \midrule
  ViT-B/16 & 39.1 & 89.3 & 91.6 & 68.7 & 46.0 & 54.1 & 70.4 & 89.2 & 68.6 & 56.0 & 88.9 & 48.1 & 65.5 & 98.2 & 65.2 & 69.8 & 83.9 \\ % & 70.2 & 70.4 \\
  \multirow{2}{*}{BASIC-M} & 49.4 & 94.2 & 94.8 & 72.2 & 60.2 & 39.5 & 86.0 & 92.3 & 81.5 & 33.6 & 95.3 & 58.3 & 65.4 & 99.3 & 72.9 & 77.4 & 84.2 \\ % & \textbf{73.0} & \textbf{73.8} \\
  & \gain{10.3} & \gain{4.9} & \gain{3.2} & \gain{3.5} & \gain{14.2} & \loss{-14.6} & \gain{15.6} & \gain{3.1} & \gain{12.9} & \loss{-22.4} & \gain{6.4} & \gain{10.2} & \loss{-0.1} & \gain{1.1} & \gain{7.7} & \gain{7.6} & \gain{0.3} \\
  \midrule
  ViT-L/14-336 & 49.5 & 92.8 & 95.7 & 77.5 & 55.7 & 59.6 & 78.3 & 93.8 & 76.2 & 88.3 & 93.5 & 63.0 & 71.7 & 99.4 & 68.4 & 76.9 & 84.3 \\ % & \textbf{77.9} & 78.4 \\
  \multirow{2}{*}{BASIC-L} & 59.2 & 94.7 & 97.5 & 82.3 & 64.6 & 51.0 & 91.2 & 95.1 & 85.7 & 40.3 & 97.9 & 59.6 & 72.7 & 99.6 & 76.2 & 84.8 & 84.6 \\ % & \textbf{78.0} & \textbf{79.1} \\
  & \gain{9.7} & \gain{1.9} & \gain{1.8} & \gain{4.8} & \gain{8.9} & \loss{-8.6} & \gain{13.1} & \gain{1.3} & \gain{9.5} & \loss{-48.0} & \gain{4.4} & \loss{-3.4} & \gain{1.0} & \gain{0.2} & \gain{7.8} & \gain{7.9} & \gain{0.3} \\
  \bottomrule
\end{tabular}
} %
\captionof{table}{\label{tab:classification_all}Performances of BASIC and CLIP models~\citep{radford21learning} on 17 image classification benchmarks. The first two blocks compare models of similar numbers of weights and FLOPs. The last block compares the largest CLIP and BASIC models.}
\end{table*}

\paragraph{Other hyperparameters.} For all experiments, we train and evaluate with the image resolution of 224x224. While we can increase this resolution to gain performance~\citep{tan19efficientnet,tan21efficientnetv2,touvron2019fixing,radford21learning,jia21scaling}, we choose not to do this and instead, reserve our computational resources for scaling up our model and our batch size.
All of our other hyper-parameters can be found in Appendix~\ref{sec:hparams_table}.\looseness=-1
\iffalse
For instance, using the GradAccum techniques from Section~\ref{sec:method}, we could train our models with the contrastive batch size of 65536.
We note that while we scale up the batch size \textit{not} for the training speed, but for its benefit to the quality of the resulting model.
Therefore, we do not increase our learning rate accordingly with the batch size. In fact, for BASIC-M and BASIC-L, our preliminary attempts show that a peak learning rate larger than $2.5 \times 10^4$ make our model diverge or otherwise unstable.
\fi

\subsection{\label{sec:image_classification}Results on Image Classification Benchmarks}
We first present the zero-shot transfer performance of our BASIC models. We compare our models BASIC-\{S,M,L\} to CLIP models with similar computational budgets~\citep{radford21learning} on 17 natural image classification datasets. Details about these datasets can be found in Appendix~\ref{sec:dataset_details}.

Zero-shot transfer models require textual prompts, which we take from CLIP~\citep{radford21learning} for consistent comparison. We suspect that using prompts which are tuned for our models can further improve our results as shown in~\citep{lester21thepower}, because the text sequences in our training data have a different distribution from the text sequences in CLIP.
% We leave this as an item for future work.

Table~\ref{tab:classification_all} shows the comparison.
From the table, it can be seen that BASIC models conclusively outperform CLIP models of the same computational budgets. Specifically, BASIC models demonstrate higher accuracy than CLIP models on 13 out of 17 datasets.
On the Oxford IIIT Pets dataset, BASIC-L achieves 97.9\% mean per-class recall which sets a new state-of-the-art, despite having never seen any training images from the dataset.
On ther other hand, BASIC models have low accuracy on EuroSAT, MNIST, and PCam. MNIST is where BASIC models perform worst, where the highest accuracy is only 40.3\%. We discuss these failure cases further in Section~\ref{sec:limitations}.

\subsection{\label{sec:robustness}Results on Robustness Benchmarks}
Despite the convincing accuracy of modern deep learning models on ImageNet, concerns have been raised about their robustness~\citep{szegedy2013intriguing}.
These concerns arise from a common failure mode of ImageNet-trained models: subtle changes to their input images, which are imperceptible to humans, can wildly alter their predictions with high confidence, \eg,~from ``golden retriever'' into ``goldfish''.\looseness=-1

In CLIP,~\citet{radford21learning} have studied certain aspects of this failure mode. They have not drawn a definitive conclusion whether to attribute such failures to deep learning, ImageNet, or a combination of them. Instead, they cautioned against generalizing ``too far from [their] initial findings''.
% However, these authors also suggest that training models on large-scale dataset using task-agnostic objective, and then deploying these models in a zero-shot manner results in more robust systems.

Here we advance CLIP's study on the robustness of zero-shot models in two aspects. First, we analyze our BASIC models presented previously in Section~\ref{sec:image_classification} and reaffirm that zero-shot models are indeed more robust than their ImageNet-trained counterparts.
Second, we perform an experiment which suggests that ImageNet's labeled training examples \textit{might be} responsible for making ImageNet-trained models less robust.
Similar to CLIP's authors, we caution readers that our experiment presents a correlation, not a causal analysis.
In other words, we do not attribute the lack of robustness in ImageNet-trained models to the dataset.
% \hieu{Someone please check the wordings here.}

\begin{figure}[tb!]
\centering
\includegraphics[width=0.45\linewidth]{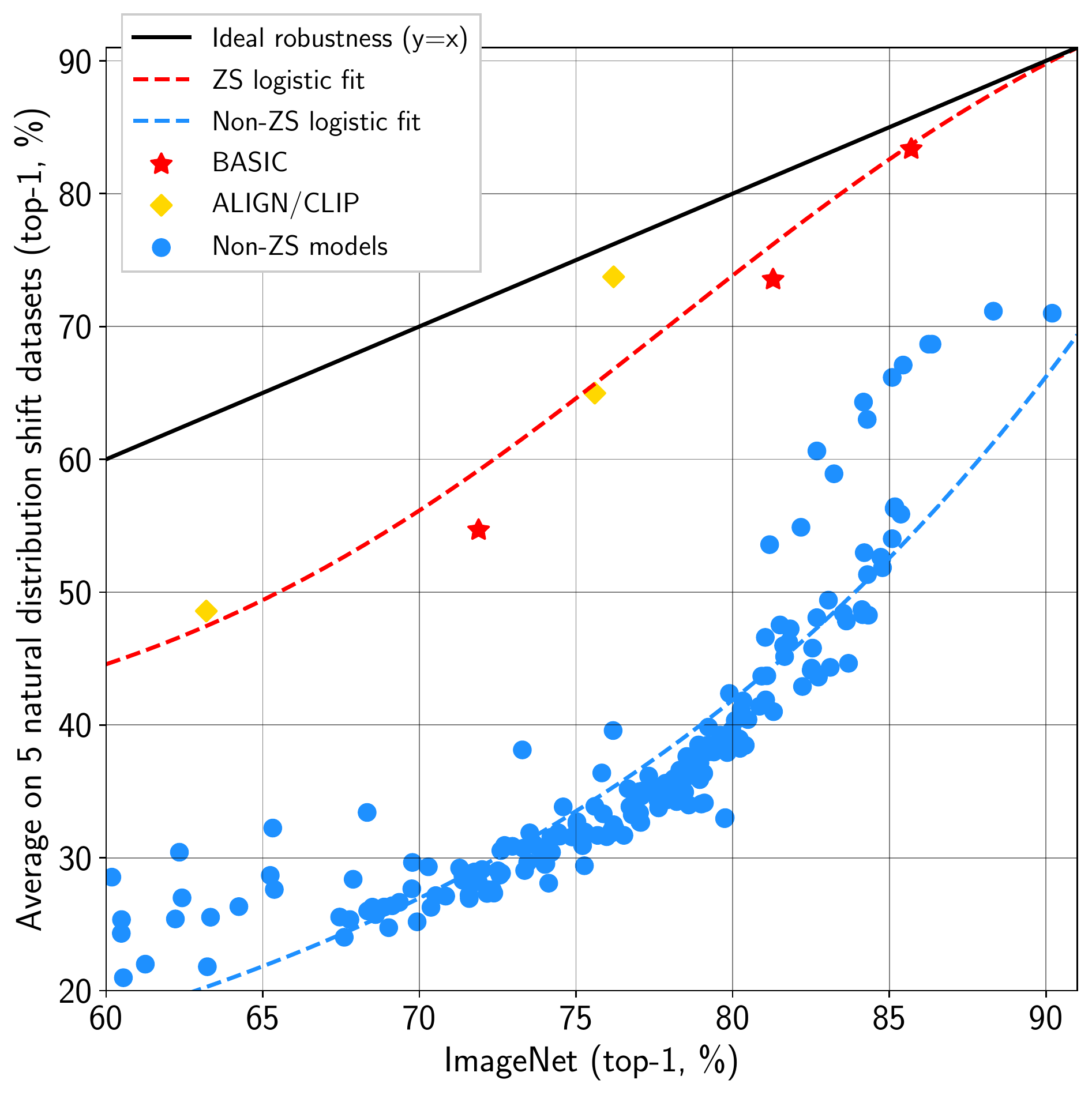}
\captionof{figure}{\label{fig:imagenet_robustness}Top-1 accuracy on ImageNet vs. average top-1 accuracy on 5 robustness benchmarks. Zero-shot models (red stars and yellow rhombuses) have significantly higher effective robustness~\citep{taori2020measuring} compared to ImageNet-trained models (blue dots).}
\end{figure}

\paragraph{More accurate zero-shot transfer models are also more robust.} We evaluate BASIC-\{S,M,L\} models from Section~\ref{sec:image_classification} on 5 robustness benchmarks derived from ImageNet: ImageNet-A~\citep{hendrycks2019nae},
ImageNet-R~\citep{hendrycks2020many},
ImageNet-V2~\citep{recht2019imagenet},
ImageNet-Sketch~\citep{wang2019learning},
and ObjectNet~\citep{barbu2019objectnet}.
These benchmarks have images in all or a subset of the 1000 ImageNet classes, but their inputs are selected from certain natural distribution shifts, which can cause ImageNet-trained models to make many more mistakes.
Our numerical results are highlighted in Table~\ref{tab:highlights} from Section~\ref{sec:intro}.
To visualize the data trend, in Figure~\ref{fig:imagenet_robustness}, we plot the accuracy of zero-shot models -- BASIC, CLIP~\citep{radford21learning}, and ALIGN~\citep{jia21scaling} -- and of 200 ImageNet-trained models collected by~\citet{taori2020measuring}.

The data points from our BASIC models extend the prediction from CLIP: zero-shot transfer models have a higher \textit{effective robustness}~\citep{radford21learning,taori2020measuring}, \ie~they have higher robustness than ImageNet-trained models with the same ImageNet accuracy.
To extrapolate from this trend, we fit a logistic curve (red dashes) to the zero-shot accuracy and robustness of zero-shot transfer models. The plot shows that this line meets the ideal robustness line at about 91\% on the x-coordinate.
In other words, our plot predicts that a model which achieves about 91\% \textit{zero-shot} accuracy on ImageNet, \ie,~just slightly better than the state-of-the-art ImageNet-trained model~\citep{dai21coatnet}, will also achieve the ideal robustness.

% \citet{taori2020measuring} discuss two robustness metrics: effective robustness captures the accuracy improvement above the prediction based on observed robustness data, and relative robustness captures any accuracy improvement on the out-of-distribution data.
% \citet{radford21learning} speculate that zero-shot models tend to have a higher effective robustness compared to ImageNet-trained models.
% In Figure~\ref{fig:imagenet_robustness}, we see that BASIC models further reaffirms this speculation: the trend that all the yellow and red markers are far above the polynomial-fit curve for the blue dots, hence reaffirming the conclusion.

\begin{figure*}[tb!]
\centering
\includegraphics[width=0.99\linewidth]{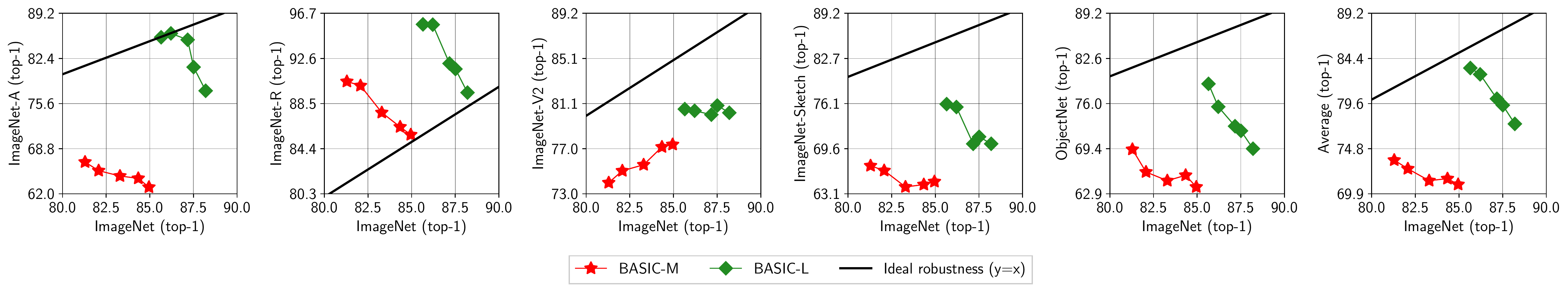}
\captionof{figure}{\label{fig:imagenet_ft}Top-1 accuracy of BASIC models on ImageNet and on 5 robustness benchmarks. In all cases, as the BASIC models are trained on more ImageNet labeled data (1\%, 10\%, 20\%, and 50\%), their ImageNet accuracy significantly increase, but their accuracy on the robustness benchmarks increase much less, or decrease.\looseness=-1}
\end{figure*}
\paragraph{ImageNet-finetuned models are less robust.} We now study the effect of ImageNet's labeled data on our models.
We take the converged BASIC-\{S,M,L\} checkpoints from Section~\ref{sec:image_classification} and continue to train them on 1\%, 10\%, 20\%, and 50\% of ImageNet's labeled examples.
Note that we continue training these checkpoints using the contrastive loss, where the names of ImageNet classes are utilized as text sequences accompanying their images. This is different from CLIP's linear probing approach, which we do not perform to avoid potential confounding factors from our study, \eg~linear classifiers might behave differently from our zero-shot transfer classifiers.
We then compare the accuracy of these finetuned models on ImageNet and on the 5 robustness benchmarks. The results are visualized in Figure~\ref{fig:imagenet_ft}.\looseness=-1

The figure shows a clear trend: as our model learns from \textit{more} labeled ImageNet data, they become more accurate on ImageNet, but these gains do not carry over to the robustness benchmarks.
Specifically, with the exception of ImageNet-V2, for which the accuracy of finetuned models stay the same (for BASIC-L) or slightly increase (for BASIC-M), for all other robustness benchmarks, the finetuned models suffer from significant performance drops.
In the extreme case, 3\% accuracy gain on ImageNet leads to 8.3\% accuracy drop for ImageNet-R.\looseness=-1

What makes our finetuned models less robust? A quick glance at our results might lead to the superficial conclusion that our models have overfit, as our finetuning sets are a lot smaller than ALIGN and JFT.
However, this overfitting theory does not explain the trend observed in Figure~\ref{fig:imagenet_ft}: training on \textit{more} labeled ImageNet data makes our models \textit{less} robust.
We hope our observation invites further causal analysis on the effects of ImageNet's labeled data.
% Here, we attribute the reason for this trend to the labeled training examples in ImageNet.
% Further in-depth discussions can be found in Appendix~\ref{sec:robustness}.

\section{\label{sec:ablation}Ablation Study}
\subsection{\label{sec:batchsize}The Importance of  Batch Size Scaling}

To demonstrate the role of large batch sizes, we conduct several controlled experiments for BASIC-S and BASIC-M on ALIGN.
% We use ALIGN instead of ALIGNv2 to rule out the effects that the larger dataset might have on the overall training process, \eg,~larger datasets might require longer training, or the models under consideration might underfit.
For both BASIC-S and BASIC-M, we fix all hyperparameters as shown in Table~\ref{tab:hparams}, but vary the batch size and the number of training steps.
Models that are trained with larger batch sizes are trained with fewer steps to guarantee that they ``see'' the same number of examples.
Table~\ref{tab:batchsize} presents the ImageNet top-1 zero-shot accuracy of all models at the end of their training, and Figure~\ref{fig:batchsize} visualizes their entire validation accuracy curves.\looseness=-1

\begin{table}[h!]
\begin{minipage}{0.6\linewidth}
  \includegraphics[width=\linewidth]{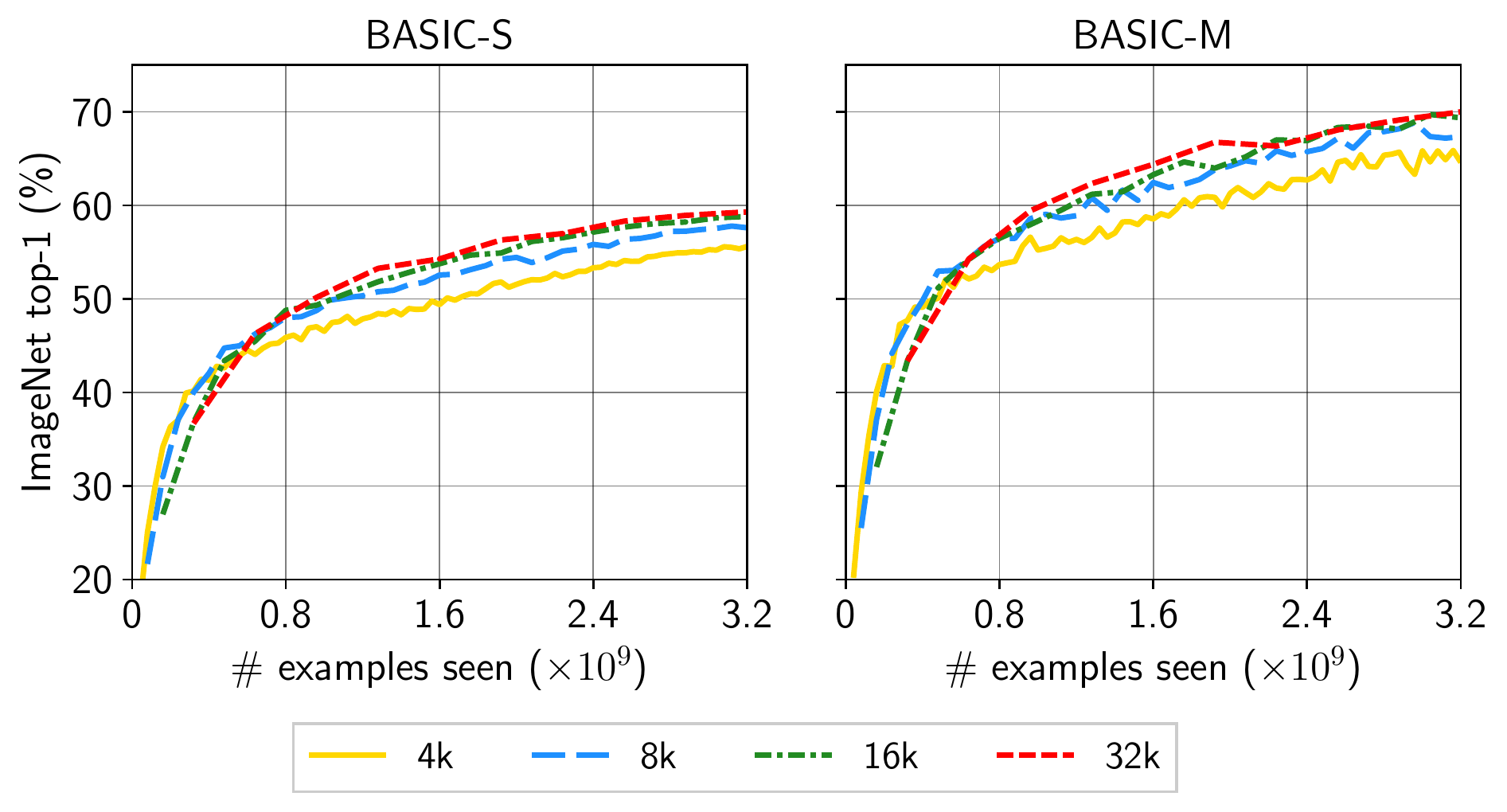}
  
  \captionof{figure}{\label{fig:batchsize}ImageNet held-out validation accuracy curves with different batch sizes. Models with smaller batch sizes are trained for more steps to ensure a fair comparison. The comparison shows that despite seeing the same number of training examples, models with larger batch sizes reach higher performances than models with more training steps. Image best viewed in color.\looseness=-1}
\end{minipage}\hfill\begin{minipage}{0.38\linewidth}
  \resizebox{\linewidth}{!}{ %
    \begin{tabular}{cccc}
      \toprule
        \textbf{Batch size} & \textbf{Steps} & \textbf{BASIC-S} & \textbf{BASIC-M} \\
      \midrule
        4096  & 800K & 55.6 & 64.8 \\
        8192  & 400K & 57.6 & 67.7 \\
        16384 & 200K & 58.8 & 69.4 \\
        32768 & 100K & \textbf{59.3} & \textbf{70.1} \\
      \bottomrule
    \end{tabular}
  } %
  \captionof{table}{\label{tab:batchsize}Top-1 ImageNet accuracy at the end of the training for our BASIC-\{S,M\} models trained with different batch sizes and numbers of training steps. All models are trained for the same number of epochs, but models trained with larger batch sizes has a higher accuracy.}
\end{minipage}
\end{table}

Table~\ref{tab:batchsize} and Figure~\ref{fig:batchsize} both suggest that training for more steps cannot equalize the benefit of large batch sizes.
This phenomenon is consistent with the observation from SimCLR~\citep{chen20simclr,chen2020big}: large batch sizes help contrastive learning.
% However, compared to SimCLR, our analysis shows a different point of saturation for the benefit of large batch sizes.
% Specifically, SimCLR, the authors observe that after at the batch size of 8192, increasing the batch size hardly improves their model's quality.
SimCLR observes that the benefit of large batch sizes saturate at 8192.
In contrast, our results in Table~\ref{tab:batchsize} and Figure~\ref{fig:batchsize} show that lager batch sizes continue to benefit our models until 32768, and even until 65536 as in Section~\ref{sec:image_classification}.
We suspect that the benefits for large batch sizes do not saturate because our dataset size and model size are both larger than those of SimCLR, \eg~ALIGN with 1.7B examples compared to ImageNet with 1M examples, and BASIC-\{S, M\} compared to ResNet-\{50,101,152\}.
This comparison suggests the benefits of our method -- combined scaling.
% in our experiments because our models and datasets are larger than SimCLR's.
% In particular, ALIGN has 1.7B examples -- 1700 times larger than ImageNet and BASIC-\{S,M\} both have far larger capacity than ResNet-50, as used in SimCLR.

\subsection{Data Scaling, Model Scaling, and Pretraining}

\begin{figure}[htb!]
\centering
\includegraphics[width=0.6\linewidth]{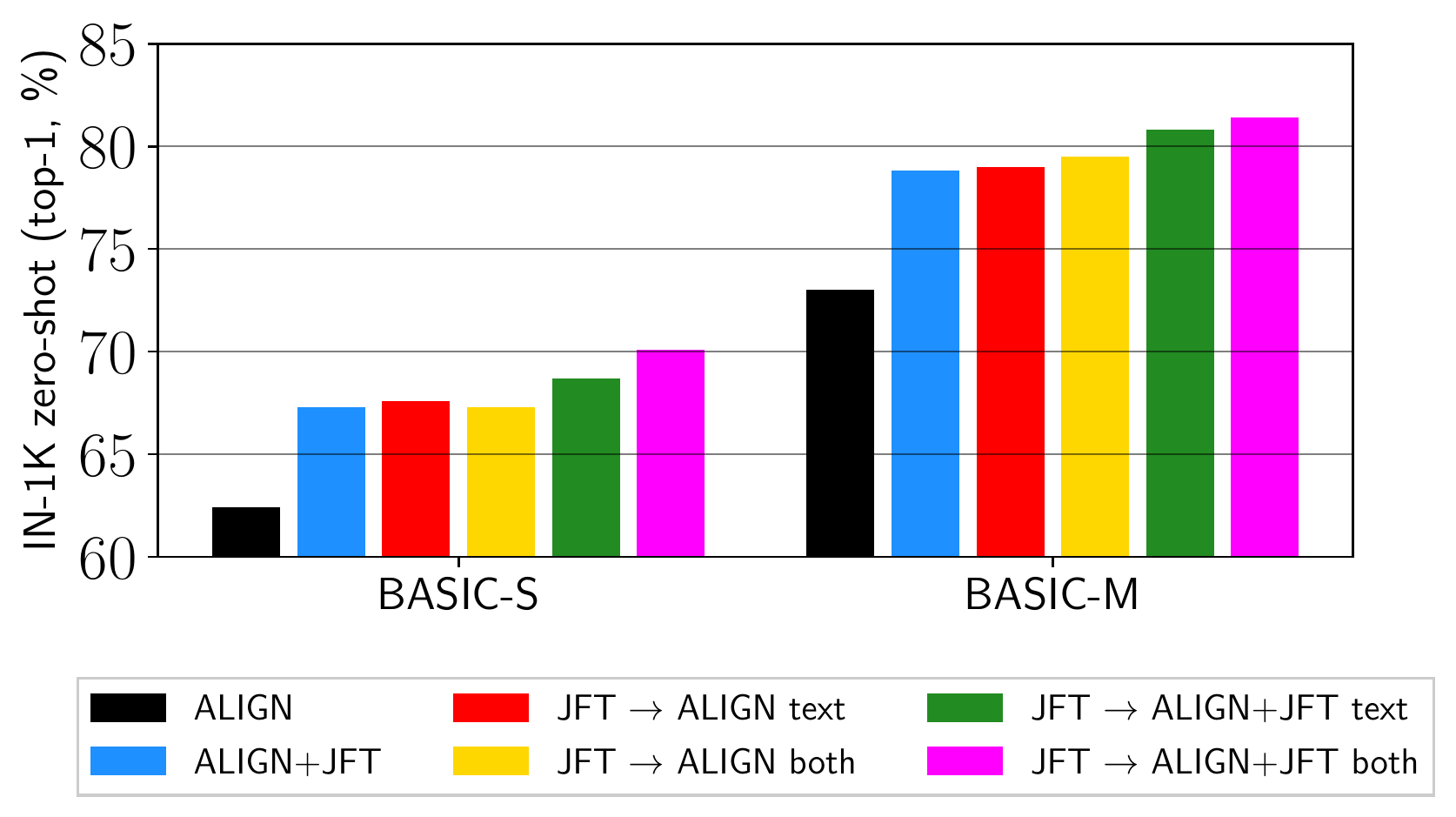}
\captionof{figure}{\label{fig:pretraining}Break-down contributions of data scaling and model scaling for BASIC-S and BASIC-M. Shown are the ImageNet top-1 accuracy of our BASIC-\{S,M\} models under different training settings. Models trained from scratch on ALIGN+JFT has almost the same performance with models pretrained on JFT and then finetuned on ALIGN or on ALIGN+JFT. Models that are pretrained and then have both their image and text encoders finetuned reach the highest accuracy. Figure best viewed in color.}
\end{figure}
We now study the benefits of other scaling dimensions, data and model scaling, on the quality of our models. We also study pretraining as an alternate training procedure to contrastive learning.
We train BASIC-\{S,M\} models in 6 different settings and plot their final top-1 ImageNet accuracy in Figure~\ref{fig:pretraining}. Below, we compare and analyze the settings.\looseness=-1

First, BASIC-S and BASIC-M respectively gain 5.3\% and 5.8\% accuracy when we expand the contrastive training dataset from ALIGN to ALIGN+JFT.
These gains, albeit large, are smaller than the gain by enlarging the model size, \eg,~11.7\% when going from BASIC-S to BASIC-M.

Next, we study the effects of pretraining image encoders on JFT.
As can be seen from Figure~\ref{fig:pretraining}, models whose image encoders are pretrained on JFT and whose text encoders are subsequently trained on ALIGN, \ie,~the red bars, have similar performances with models trained from scratch on ALIGN+JFT, \ie,~the blue bars.
Their similar accuracy suggest that the training losses -- softmax cross-entropy or contrastive -- have a much smaller effect than the datasets.
In other words, when given the same dataset, the image encoders in BASIC models learn to become equally good, regardless of their loss functions.

To our surprise, training the text encoders for JFT-pretrained image encoders \textit{on ALIGN+JFT} gains 1\% for BASIC-S and 1.8\% for BASIC-L, compared to training these text encoders on ALIGN. We suspect that these gains come from better representations for the textual prompts, since the models trained on ALIGN+JFT also sees the textual prompts which consist of clean JFT class names. However, this speculation needs a more thorough study to understand.

Finally, we find that if we take a converged model whose image encoder is pretrained on JFT and whose text encoder is trained on ALIGN+JFT, then we continue to train \textit{both} its image encoders and text encoders at a small learning rate.
This extra training phase gains us 1.4\% ImageNet accuracy for BASIC-S, 0.6\% for BASIC-M, and 0.4\% for BASIC-L (not shown in this section).

% \paragraph{Effects of prompts.}
% \paragraph{Plots of class names.}

\section{\label{sec:limitations}Limitations}
Despite the strong results of our zero-shot transfer classifier, especially on natural image classification tasks, they inevitably have their shortcomings. In this section, we discuss the problems that we find with our BASIC models.

\paragraph{Zero-shot transfer models do not perform well on test sets that are underrepresented in the training datasets.} We emphasize the failures of BASIC on two test sets where BASIC models are much worse than CLIP models: EuroSAT, MNIST, PatchCamelyon (PCam) (see Table~\ref{tab:classification_all} from Section~\ref{sec:image_classification}).
Here, we summarize that BASIC models fail on MNIST and PCam because our training datasets ALIGN and JFT have relatively few images of handwritten digits and of lymph nodes, which are the domain of these datasets.
Compared to MNIST and PCam, BASIC models do better on EuroSAT which consist of satellite land images, but their accuracy is lower than that of CLIP models. This is because the class names for these satellite images are not very descriptive to BASIC models. More analysis for these failures are in Appendix~\ref{sec:failure}.

% the reason for BASIC's failures. Our training dataset ALIGN+JFT contains mostly natural images and thus has few handwritten digit images similar to MNIST test images, and very few lymph node images similar PCam test images.\looseness=-1

\paragraph{Zero-shot transfer requires prompt engineering.} In this paper, we use the prompts from CLIP~\citep{radford21learning} to make our results comparable to previous works.
% Prompt engineering can significantly affect the results of our models. 
In Appendix~\ref{sec:failure}, we present examples which show that prompts that are badly chosen or adversarially chosen can hurt the accuracy of zero-shot transfer models by flipping their predictions.
These examples suggest that prompt engineering is an important research topic to make zero-shot transfer models robust and reliable, but the topic is of out of the scope of this paper.

\paragraph{Combined scaling is expensive.}
% Finally,  scaling up models, data and batch sizes as in our case is expensive.
As reported in Appendix~\ref{sec:compute_cost}, the hardware and training time for our models are not small.
Despite the training cost, we can use the models in this paper without any finetuning, and hence avoid the finetuning cost.
We hope that future research can reduce our models' training expense, \eg,~larger accelerator memory can save the extra re-materialization steps.

\section{\label{sec:conclusion}Conclusion}
Zero-shot transfer learning represents a new paradigm where pretrained models can be used directly for downstream applications without collecting any application-specific data.
However, in order to become practical for real-world applications, zero-shot transfer models need to bridge the accuracy gap to supervised and semi-supervised models.

In this paper, we presented combined scaling techniques that significantly boost the performance of zero-shot transfer models.
We show that scaling in the data size, the model size, and the batch size all improves the final model's accuracy and robustness.
To overcome the memory limit arising from combined scaling, we devise a simple gradient accumulation method based on re-materialization.

{
\setlength{\bibsep}{0pt}
\small
\bibliography{egbib}
}

\newpage
\appendix
\onecolumn

\section{\label{sec:model_size}Model sizes}
In our preliminary experiments, we experimented with different model sizes. Table~\ref{tab:modelsize} presents the final, most compute-to-performance efficient model sizes, which we use throughout the paper.

\begin{table}[htb!]
\centering
\resizebox{0.9\linewidth}{!}{%
  \begin{tabular}{lcccccccc}
    \toprule
    & \multicolumn{3}{c}{\textbf{Image model}}
    & \multicolumn{5}{c}{\textbf{Text model} }\\
    \cmidrule(lr){2-4} \cmidrule(lr){5-9}
    & Model~\citep{dai21coatnet} & \#Params & \#FLOPs 
    & \#Layers & HiddenDim & HeadDim & \#Params & \#FLOPs \\
    \midrule
    BASIC-S & CoAtNet-0 &  25M &  4.2B &  6 & 1024 &  64 & 108M & 10.7B \\
    BASIC-M & CoAtNet-3 & 168M & 34.7B & 12 & 1024 & 128 & 184M & 49.4B \\
    BASIC-L & CoAtNet-7 & 2.4B &  495.8B & 12 & 2048 & 128 & 670M & 212.6B \\
    \bottomrule
  \end{tabular}
}%
\caption{\label{tab:modelsize}Model sizes. For the image models, all specifications can be found from the model names in~\citet{dai21coatnet}.}
\end{table}

\iffalse
\begin{figure}[htb!]
\centering
\includegraphics[width=0.7\linewidth]{figures/scaling.pdf}
\caption{\label{fig:scaling_dims}ImageNet top-1 accuracy of different models trained at different data sizes and different batch sizes.}
\end{figure}
\fi

\section{\label{sec:hparams_table}Hyperparameters and other implementation details}
Our training and evaluation code will eventually be released. Here, we summarize a few important details. All of our hyper-parameters are in Table~\ref{tab:hparams}.

\paragraph{No regularization.}
Other than the decoupled weight decay in AdaFactorW, we do not use any other regularization technique. In fact, we find that with BASIC-S and BASIC-M, if we add other forms of regularization such as stochastic depth~\citep{huang17deep} or dropout~\citep{srivastava14dropout}, our ImageNet top-1 accuracy drops substantially. This suggests that our datasets are very large and perhaps in such situation, regularization techniques do more harm than good by causing optimization difficulty to our models.\looseness=-1

Another important effect of \textit{not} using regularization in our training framework is to make the re-materialization steps in Section~\ref{sec:gradaccum} consistent. If we apply random perturbations to our forward passes, \eg~by skipping layers like in stochastic depth or by setting random values to zeros, then two forward passes for re-materialization (see Lines 2-5 and 11-14 in Algorithm~\ref{alg:grad_accum}) will compute two different passes. While we could treat such difference as a form of regularization noise, our early experiment show that with dropout-like regularizations, our training loss stays relatively large throughout the course of training. This observation suggests that the noise causes some optimization difficulty to our models, so we opt not to use any dropout-like regularization.\looseness=-1

\begin{table}[htb!]
\centering
\resizebox{0.7\linewidth}{!}{%
  \begin{tabular}{lcccc}
    \toprule
    & \multicolumn{2}{c}{\textbf{BASIC-S}}
    & \multicolumn{2}{c}{\textbf{BASIC-\{M,L\}}} \\
    & Pretraining & Contrastive
    & Pretraining & Contrastive \\
    \midrule
    Optimizer  & AdaFactorW & AdaFactorW & AdaFactorW & AdaFactorW \\
    Batch size & 16384 & 65536 & 16384 & 65536 \\
    Training steps & 500K & 500K & 1.2M & 500K \\
    Warm-up steps & 25K & 25K & 25K & 25K \\
    Max learning rate & 1e-3 & 1e-3 & 4e-4 & 2.5e-4 \\ 
    Min learning rate & 1e-5 & 1e-5 & 2e-5 & 1e-5 \\ 
    Learning decay schedule & Cosine & Cosine & Linear & Cosine \\ 
    Weight decay & 0.005 & 0.0025 & 0.01 & 0.0025 \\ 
    \bottomrule
  \end{tabular}
}%
\caption{\label{tab:hparams}Hyperparameters all of our experiments.}
\end{table}

\section{\label{sec:dataset_details}Evaluation Datasets Details}
Here, we present the details of the datasets which we use to evaluate our BASIC models in Section~\ref{sec:image_classification}. It is worth noting that \textit{not} all these datasets use the accuracy as the performance metric. This is because these datasets have a certain level of imbalance between their classes, as well as other properties that make them accuracy not the best suitable metric for them. For instance, the dataset Caltech-101 has a class called ``Background'' which refers to \textit{any} image that does not belong to its predefined 101 classes. One certainly cannot come up with a textual description that describes this ``class''. As such, Caltech-101 is evaluated using mean per-class recall. Details about other datasets are in Table~\ref{tab:datasets}.

\begin{table}[htb!]
\centering
\resizebox{0.99\linewidth}{!}{ %
\begin{tabular}{lccccc}
  \toprule
  \textbf{Dataset} & \textbf{Reference} &
  \textbf{Abbreviation in Table~\ref{tab:classification_all}} &
  \textbf{\#Classes} &
  \textbf{Test size} &
  \textbf{Evaluation metric} \\
  \midrule
  ILSVRC-2012, \ie,~ImageNet & \citep{russakovsky09imagenet} & ImageNet & 1000  & 50000 & accuracy \\
  ImageNet-A & \citep{hendrycks2019nae} & N/A & 1000 & 7500 & accuracy \\
  ImageNet-R & \citep{hendrycks2020many} & N/A & 1000 & 30000 & accuracy \\
  ImageNet-V2 & \citep{recht2019imagenet} & N/A & 1000 & 10000 & accuracy \\
  ImageNet-Sketch & \citep{wang2019learning} & N/A & 1000 & 50889 & accuracy \\
  ObjectNet & \citep{barbu2019objectnet} & N/A & 1000 & 18574 & accuracy \\
  CIFAR-10 & \citep{krizhevsky09learning} & CIFAR10 & 10 & 10000 & accuracy \\
  CIFAR-100 & \citep{krizhevsky09learning} & CIFAR100 & 100 & 10000 & accuracy \\
  Birdsnap & \citep{berg2014birdsnap} & Birdsnap & 500 & 2443 & accuracy \\
  Describable Textures & \citep{cimpoi14describing} & DTD & 47 & 1880 & accuracy \\
  Oxford Flowers-102 & \citep{nilsback08automated} & Flowers & 102 & 6149 & mean per-class recall \\
  Food-101 & \citep{food101} & Food101 & 101 & 25250 & accuracy \\
  Caltech101 & \citep{feifei04learning} & Caltech101 & 102 & 6084 & mean per-class recall \\
  Oxford IIIT-Pets & \citep{oxfordpets} & IIIT-Pets & 37 & 3669 & mean per-class recall \\
  MNIST & \citep{lecun2010mnist} & MNIST & 10 & 10000 & accuracy \\
  EuroSAT & \citep{helber2018introducing} & EuroSAT & 10 & 27000 & accuracy \\
  PatchCamelyon & \citep{veeling2018rotation} & PCam & 2 & 32768 & accuracy \\
  RESICS45 & \citep{cheng17remote} & RESICS45 & 45 & 31500 & accuracy \\
  STL10 & \citep{coates2011stl10} & STL10 & 10 & 8000 & accuracy \\
  SUN397 & \citep{xiao10sun} & SUN397 & 397 & 21750 & accuracy \\
  UCF101 & \citep{soomro12ucf} & UCF101 & 101 & 3783 & accuracy \\
  Pascal VOC 2007 Classification & \citep{everingham07pascal} & VOC2007 & 20 & 4952 & 11-points mAP \\
  \bottomrule
\end{tabular}
}
\captionof{table}{\label{tab:datasets}Details of the datasets used in this paper to evaluate BASIC models. The evaluation results are presented in Table~\ref{tab:highlights} and Table~\ref{tab:classification_all}.}
\end{table}

\section{\label{sec:robustness_further}Further Discussion on Robustness}
In Section~\ref{sec:robustness}, we present a surprising result: finetuning converged BASIC checkpoints on \textit{more} ImageNet labeled data leads to \textit{worse} robustness results. The metric for robustness in Section~\ref{sec:robustness} is the \textit{average} top-1 accuracy of the finetuned models on 5 robustness benchmarks derived from ImageNet~\citep{hendrycks2019nae,hendrycks2020many,recht2019imagenet,barbu2019objectnet,wang2019learning}. It turns out that each of these benchmarks can demonstrate slightly different results for the finetuned models. Here, we discuss such benchmarks.

\paragraph{ImageNet-V2~\citep{recht2019imagenet}.} This dataset is collected in a process that closely follows the process to collect and annotate the images in the standard ILSVRC-2012 validation set, which is typically referred to as ``ImageNet'' in the literature (and our paper as well). As such, gains observed on ImageNet often transfer to ImageNet-V2. Recent works such as EfficientNets~\citep{tan19efficientnet,tan21efficientnetv2} or ViT~\citep{dosovitskiy21vit} also demonstrate the similar trend. For our experiment in Section~\ref{sec:robustness}, BASIC-M's robustness accuracy improves along with its ImageNet accuracy, following this trend. However, BASIC-L's robustness does not. We suspect this trend is because BASIC-L's learning capacity is larger than that of BASIC-M, so BASIC-L picks up more ``spurious'' patterns from ImageNet, making it less robust than BASIC-M.

\paragraph{ImageNet-R~\citep{wang2019learning}.} ImageNet-R is a special robustness dataset in our study. Not only of our BASIC models but also other zero-shot models -- CLIP and ALIGN -- are more accurate on ImageNet-R than they are on ImageNet (see Table~\ref{tab:highlights}).
These data points alone would suggest that ImageNet-R is somewhat easier than ImageNet, until we look at the significant accuracy drops for other methods on ImageNet-R. For instance, Noisy Student~\citep{xie2020self} and Meta Pseudo Labels~\citep{pham2021meta} respectively achieve only 74.9\% and 72.7\% accuracy on ImageNet-R, despite their accuracy of 88.4\% and 90.2\% on ImageNet ILSVRC-2012.
The real reason for such discrepancy in ImageNet-R is that ImageNet-R is collected by selecting the ImageNet classes from visual art pieces, such as paintings, cartoons, graffiti, origami, and sculptures.
These art pieces are often displayed in a clean environment, free of noises such as multiple classes per image, making the images easier to recognize.
As such, BASIC, CLIP, and ALIGN, all perform better on ImageNet-R.
However, ImageNet-R images have a drastically different distribution compared to ImageNet \textit{labeled} training images, as they are respectively art images and natural images. This is why ImageNet-trained models display a much lower accuracy on ImageNet, compared to zero-shot models.

\paragraph{The case of ObjectNet~\citep{barbu2019objectnet}.} From Table~\ref{tab:highlights}, it can be seen that BASIC model's improvement over ALIGN and CLIP on Object is significantly lower than others on other benchmarks, \ie,~6.6\% compared to more than 8\% (except for ImageNet-R, for which the accuracy of all models are saturated at over 90\%). We find out the reason is that, even though ObjectNet has images from the same classes with ImageNet, these objects turn out to have their own more descriptive names, \eg~the class name ``chairs'' in ImageNet could be ``chairs by [viewpoint]'' or ``chairs with [background]''. As we later show in Section~\ref{sec:failure}, using different class names and prompts can affect our results. This effect has also been observed in CLIP~\citep{radford21learning}. Here, we take the same class names and prompts for ImageNet and use them for ObjectNet. We suspect that using ObjectNet-specific class names and prompts can improve our result.

\section{\label{sec:compute_cost}Computational Cost}
All of our models are implemented in TensorFlow~\citep{abadi2016tensorflow} and trained on Tensor Processing Units (TPUs~\citep{jouppi2017tpu}). Our BASIC-S and BASIC-M models are all trained on TPUv3 chips, while our BASIC-L models are trained on TPUv4 chips. These TPUv4 chips in their MegaCore mode can offer 32GB of memory, out of which our BASIC-L models use 30.1GB, which means that our model essentially saturates the TPU's memory. We note that oftentimes, a small portion of TPU memory needs to be reserved for their low-level infra systems. Therefore, our BASIC-L models essentially saturate the accelerators with the largest memory currently available. Given this memory usage, we use Algorithm~\ref{alg:grad_accum} with the microbatch size $M=8192$ and the batch size $N=65536$ to train this model. Table~\ref{tab:compute_usage} summarizes the training cost for each phase of our models BASIC-\{S,M,L\} as in Section~\ref{sec:image_classification}.

\begin{table}[htb!]
\centering
\resizebox{0.7\linewidth}{!}{ %
\begin{tabular}{lcccccc}
\toprule
\multirow{2}{*}{\textbf{Model}} &
\multicolumn{2}{c}{\textbf{Pretraining}} &
\multicolumn{2}{c}{\textbf{Text Encoder}} &
\multicolumn{2}{c}{\textbf{Text \& Image Encoders}} \\
\cmidrule(lr){2-3} \cmidrule(lr){4-5} \cmidrule(lr){6-7}
& Type & Cores$\times$Days
& Type & Cores$\times$Days
& Type & Cores$\times$Days \\
\midrule
BASIC-S & TPUv3 & 0.4K & TPUv3 & 0.9K & TPUv3 & 0.3K \\
BASIC-M & TPUv3 & 1.7K & TPUv3 & 3.9K & TPUv3 & 1.2K \\
BASIC-L & TPUv4 & 6.9K & TPUv4 & 1.0K & TPUv4 & 0.8K \\
\bottomrule
\end{tabular}
} %
\captionof{table}{\label{tab:compute_usage}Computational usages to train our models. Core$\times$Days is the product of the number of training days and the number of cores used to train the models. For instance, using 2048 TPUs in 1 day equals to 2.048 Cores$\times$Days. We use this metric because sometimes, our jobs are run on different numbers of TPUs due to limited availability.}
\end{table}

\section{\label{sec:visual}Qualitative Analysis: Successful Classification Examples}
Zero-shot transfer models open the door to versatile applications. This section is dedicated to demonstrating their versatility. In Figure~\ref{fig:sample_examples}, we visualize some predictions of our best model, BASIC-L, on instances that are less expected on traditional image classification benchmarks. We come up with the text sequences and demonstrate that the model can indeed align images to the most appropriate sequence.

\begin{figure*}[h]
\centering

\begin{minipage}{0.125\linewidth}
    \includegraphics[width=1.0\linewidth, height=1.0\linewidth]{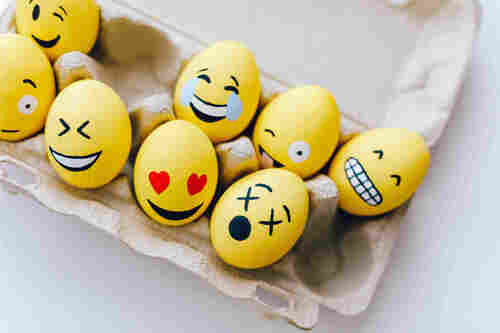}
\end{minipage}
\begin{minipage}{0.36\linewidth}
    \footnotesize
    \begin{tabularx}{\textwidth}{X|l}
    \emph{Eggs with mixed expressions}. (\cmark) & $0.944$ \\
    \emph{Emojis with mixed expressions}. & $0.032$ \\
    \emph{Happy eggs}. & $0.022$ \\
    \emph{Sad eggs}. & $0.002$ \\
    \emph{Happy emojis}. & <1e-3 \\
    \emph{Sad emojis}. & <1e-3 \\
    \end{tabularx}
\end{minipage}
\hfill
\begin{minipage}{0.125\linewidth}
    \includegraphics[width=1.0\linewidth, height=1.0\linewidth]{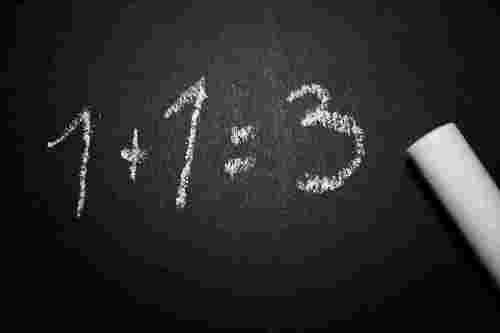}
\end{minipage}
\begin{minipage}{0.36\linewidth}
    \footnotesize
    \begin{tabularx}{\textwidth}{X|l}
    \emph{One plus one equals three}. (\cmark) & $0.468$ \\
    \emph{One plus one equals one}. & $0.264$ \\
    \emph{One plus one equals two}. & $0.240$ \\
    \emph{One minus one equals three}. & $0.014$ \\
    \emph{One minus one equals two}. & <1e-3 \\
    \emph{One minus one equals one}. & <1e-3 \\
    \end{tabularx}
\end{minipage}

\begin{minipage}{0.125\linewidth}
    \includegraphics[width=1.0\linewidth, height=1.0\linewidth]{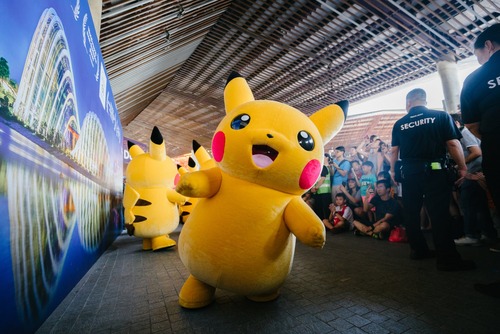}
\end{minipage}
\begin{minipage}{0.36\linewidth}
    \footnotesize
    \begin{tabularx}{\textwidth}{X|l}
    \emph{Cosplayed pikachu}. (\cmark) & $0.764$ \\
    \emph{Cosplayed charmander}. & $0.219$ \\
    \emph{Real pikachu}. & $0.012$ \\
    \emph{Cosplayed eevee}. & $0.006$ \\
    \emph{Real charmander}. & <1e-3 \\
    \emph{Real eevee}. & <1e-3 \\
    \end{tabularx}
\end{minipage}
\hfill
\begin{minipage}{0.125\linewidth}
    \includegraphics[width=1.0\linewidth, height=1.0\linewidth]{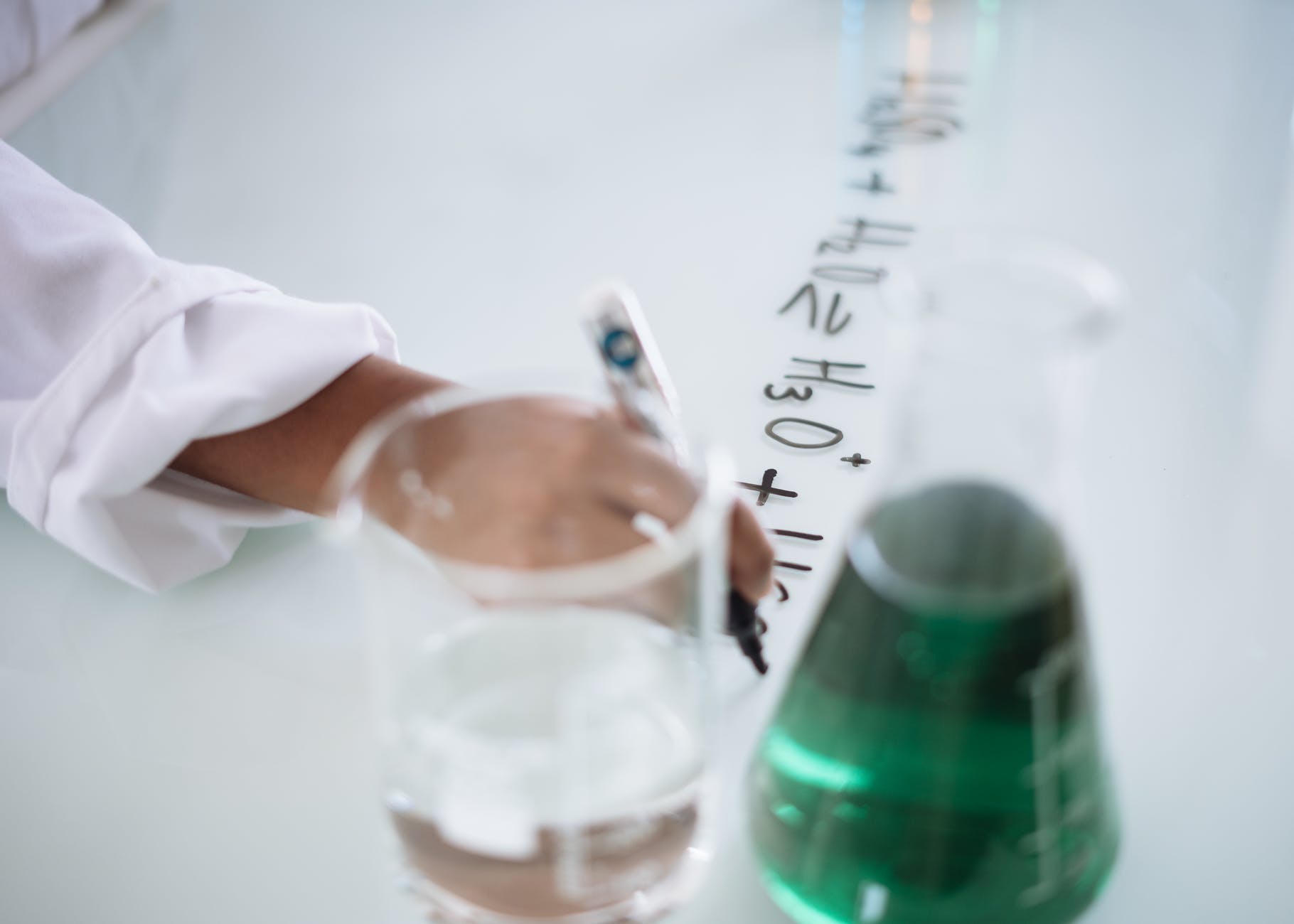}
\end{minipage}
\begin{minipage}{0.36\linewidth}
    \footnotesize
    \begin{tabularx}{\textwidth}{X|l}
     \emph{Chemistry equation on a whiteboard}. (\cmark) & $0.958$ \\
     \emph{Math equation on a whiteboard}. & $0.024$ \\
     \emph{Physics equation on a whiteboard}. & $0.012$ \\
     \emph{Chemistry equation on a paper}. & $0.005$ \\
     \emph{Math equation on a paper}. & <1e-3 \\
     \emph{Physics equation on a paper}. & <1e-3 \\
    \end{tabularx}
\end{minipage}

\begin{minipage}{0.125\linewidth}
    \includegraphics[width=1.0\linewidth, height=1.0\linewidth]{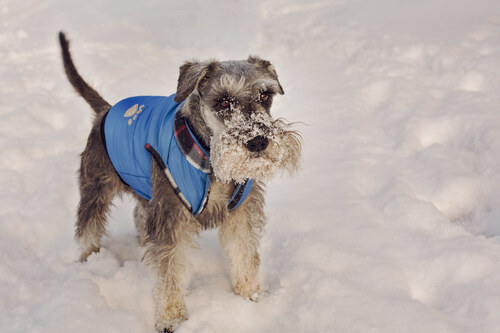}
\end{minipage}
\begin{minipage}{0.36\linewidth}
    \footnotesize
    \begin{tabularx}{\textwidth}{X|l}
    \emph{A shirari dog in cold weather}. (\cmark) & $0.961$ \\
    \emph{A shirari dog in warm weather}. & $0.038$ \\
    \emph{A corgi dog in cold weather}. & <1e-3 \\
    \emph{A shiba inu dog in cold weather}. & <1e-3 \\
    \emph{A corgi dog in warm weather}. & <1e-3 \\
    \emph{A shiba inu dog in warm weather}. & <1e-3 \\
    \end{tabularx}
\end{minipage}
\hfill
\begin{minipage}{0.125\linewidth}
    \includegraphics[width=1.0\linewidth, height=1.0\linewidth]{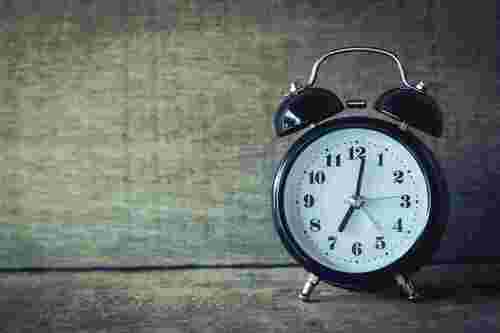}
\end{minipage}
\begin{minipage}{0.36\linewidth}
    \footnotesize
    \begin{tabularx}{\textwidth}{X|l}
    \emph{An alarm clock that reads 7:00 pm}. (\cmark) & $0.288$ \\
    \emph{An alarm clock that reads 10:00 am}. & $0.183$ \\
    \emph{An alarm clock that reads 12:00 pm}. & $0.167$ \\
    \emph{An alarm clock that reads 4:00 pm}. & $0.162$ \\
    \emph{An alarm clock that reads 7:00 am}. & $0.133$ \\
    \emph{An alarm clock that reads 2:00 pm}. & $0.067$ \\
    \end{tabularx}
\end{minipage}

\captionof{figure}{\label{fig:sample_examples}Selected classification examples from BASIC-L over unseen images.}
% \hanxiao{Are there any messages that you want to show from these examples?}}
\end{figure*}

\section{\label{sec:failure}Failure Analysis}
Most machine learning models fail in certain tests. It is important to identify such failure cases, to understand the failing causes, and if possible, to come up with fixes. Here, we first look at the test benchmarks in Table~\ref{tab:classification_all} from Section~\ref{sec:image_classification} where BASIC models perform worse than CLIP models. We identify the cause of failures for BASIC models and recommend certain fixes that can improve their performance. Then, in Section~\ref{sec:failed_cases}, we present some erroneous behaviors of BASIC models via selected examples. These examples reveal some weaknesses of zero-shot transfer models, and invite future research to improve them.

\subsection{\label{sec:failed_datasets}The benchmarks where BASIC fails}
From Section~\ref{sec:image_classification}, we see that BASIC models have particularly low performance on EuroSat~\citep{helber2018introducing}, MNIST~\citep{lecun2010mnist}, and Patch Camelyon~\citep{veeling2018rotation}. The accuracy of BASIC-L on these datasets are 51.0\%, 40.3\%, and 59.6\% respectively. For what it's worth, BASIC-L's accuracy are better than those of our smaller models, \ie,~BASIC-S and BASIC-M, so our central message in this paper -- scaling helps -- is not altered. Here, we focus on analyzing the failures of BASIC-L.

\paragraph{Patch Camelyon (PCam).} PCam is perhaps the most sensitive dataset among the three benchmarks where BASIC-L performs poorly. This dataset consists of images extracted from histopathologic scans of lymph node sections, and models are asked to make the binary prediction -- whether an input image has a cancerous lymph node or note. For such an important task, the top-1 accuracy of both BASIC-L (59.6\%) and CLIP (63.0\%) are far below the bars for practical deployments. We remark that PCam is a binary classification task, so the accuracy of BASIC-L and CLIP are just slightly above random guessing. Their poor performance, however, are quite understandable: classifying lymph nodes requires much more specific training, compared to classifying common natural images. As our training data are weakly crawled and automatically curated from the internet, without any emphasis on medical images, our BASIC-L model cannot learn enough to perform well on PCam. We suspect the same speculation also holds for CLIP, as their data collection and curation process is comparable to ours. Finally, the low accuracy of CLIP and BASIC models on PCam is an assertion that despite the benefits of zero-shot transfer models, they are not ready to be deployed to tasks that require in-domain expertise, \eg~medical knowledge.

\paragraph{EuroSAT.} This dataset consists of satellite images taken for certain types of lands. Models are asked to classify input images into one out of 10 given types of lands. The land types can be seen in Figure~\ref{fig:eurosat}. The failure of BASIC-L on EuroSAT is an example for the importance of prompt engineering in zero-shot transfer learning for image-text models. In Figure~\ref{fig:eurosat}, we show that by changing the dataset's \textit{class names} and the model's set of prompts, into words and phrases that essentially have the same meaning to humans, we can improve the accuracy of BASIC-L from 51.0\% to 55.7\%. We do not further explore the changes in class names and prompts to improve BASIC-L's performance on EuroSAT, as they belong to a different topic from the focus of this paper -- combined scaling. However, our findings on this EuroSAT dataset suggests that contrastive image-text models do not really ``understand'' texts. This is perhaps because of the low quality of the texts in our training data, unlike the millions of words from books and articles like the training data of NLP models such as BERT~\citep{devlin18bert}.

\begin{figure}[htb!]
\centering
\begin{minipage}{0.73\linewidth}
  \centering
  \includegraphics[width=\linewidth]{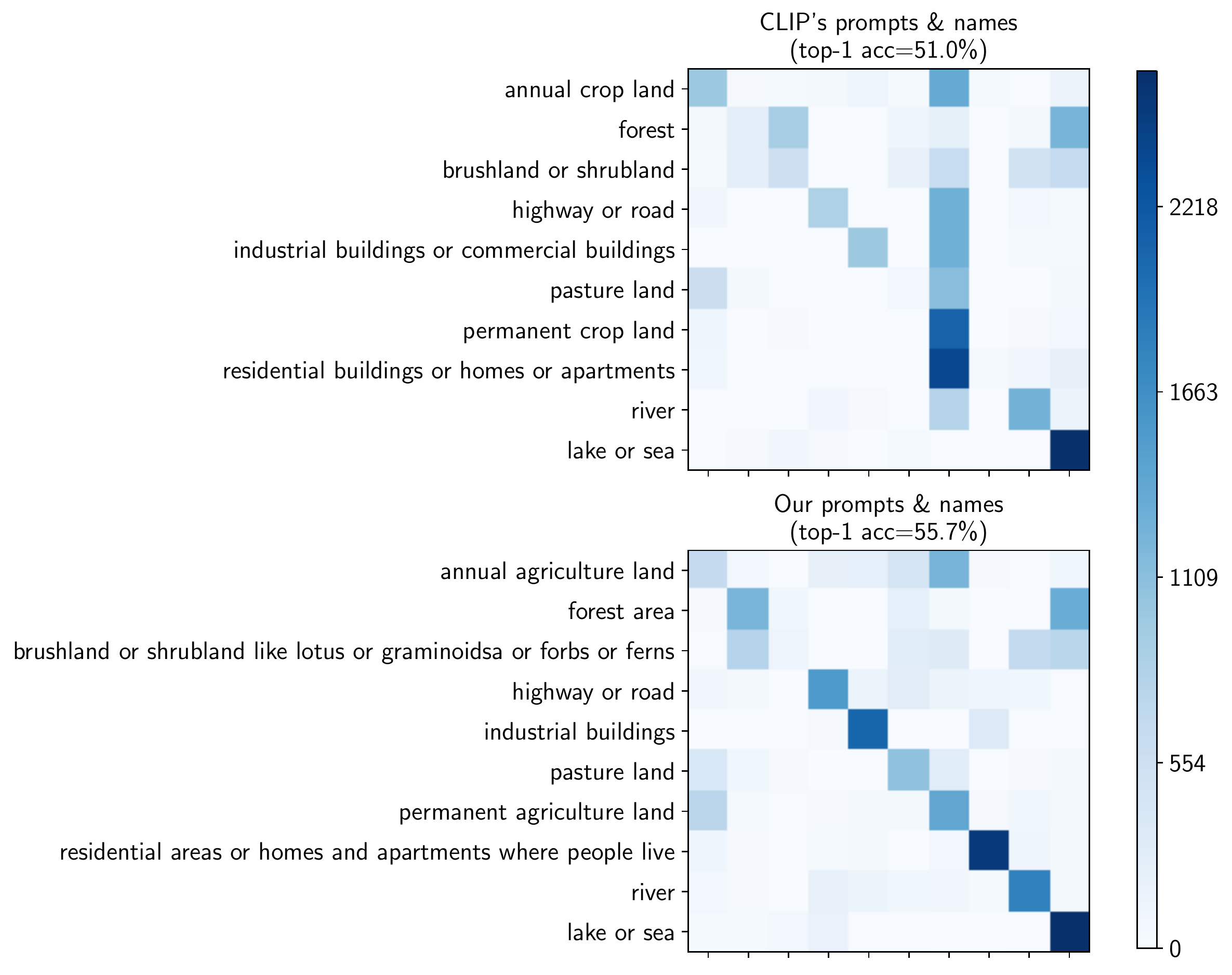}
\end{minipage}
\begin{minipage}{0.8\linewidth}
\resizebox{\linewidth}{!}{ %
\begin{tabular}{ll}
  \toprule
    \textbf{Prompts from CLIP~\citep{radford21learning}} & \textbf{Prompts we find} \\
  \midrule
    `a centered satellite photo of \{\}.', & `a centered satellite photo of \{\}' \\
    `a centered satellite photo of a \{\}.', & `a satellite photo of \{\}' \\
    `a centered satellite photo of the \{\}.', & `a photo of \{\} taken from a satellite' \\
    & `a photo of \{\} taken from the sky' \\
    & `a picture of \{\} taken from the sky by a satellite' \\
    & `a picture of \{\} taken by a satellite' \\
    & `a picture of \{\} taken by a satellite in space' \\
    & `a picture of \{\} taken by a satellite in its orbit' \\
  \bottomrule
\end{tabular}
} %
\end{minipage}
\caption{\label{fig:eurosat}Confusion matrices of BASIC-L on the EuroSAT classification dataset~\citep{helber2018introducing}. Shown are the confusion matrices obtained from zero-shot transfer learning from BASIC-L, using prompts and class names and CLIP, compared to the same model using prompts and class names that we tuned. The zero-shot top-1 accuracy with our prompts and class names are 4.7\% higher, and the confusion matrix illustrates this by showing more concentration on the diagonal.}
\end{figure}

\paragraph{MNIST.} MNIST is a classical dataset in computer vision for handwritten digit classification. Simple models can achieve more than 99.5\% accuracy, and yet BASIC-L achieves the humble 40.3\% accuracy. Unlike the case of PCam, \ie~there is not enough training data in our training dataset, for MNIST, we find that the ALIGN dataset has a fair amount of images that contain digits, either handwritten or printed. This means that the image encoder of BASIC-L has seen digit figures, and suggests that the failures might be more attributable to the text encoder, similar to the case of EuroSAT. In Figure~\ref{fig:mnist}, we show the confusion matrices of BASIC-L models with three sets of class names: using the digits such as \{`0', `1', ...\}, using the the texts such as \{`one', `two`', ...\}, and using both such as \{`0 or zero', `1 or one', ...\}. Unfortunately, we cannot improve BASIC-L's accuracy on MNIST, like we did for EuroSAT: BASIC-L's accuracy is low in all three cases, but the confusion matrices are visibly different: BASIC-L models `thinks' that many digits look like `3' for the digit-only class names, but many digits look like `1 or one' in the digit-and-text class names. Again, humans who understand languages will not make these mistakes. We think these mistakes constitute a new type or robustness failures, which we hope will invite further research.

\begin{figure}[htb!]
\centering
\includegraphics[width=0.7\linewidth]{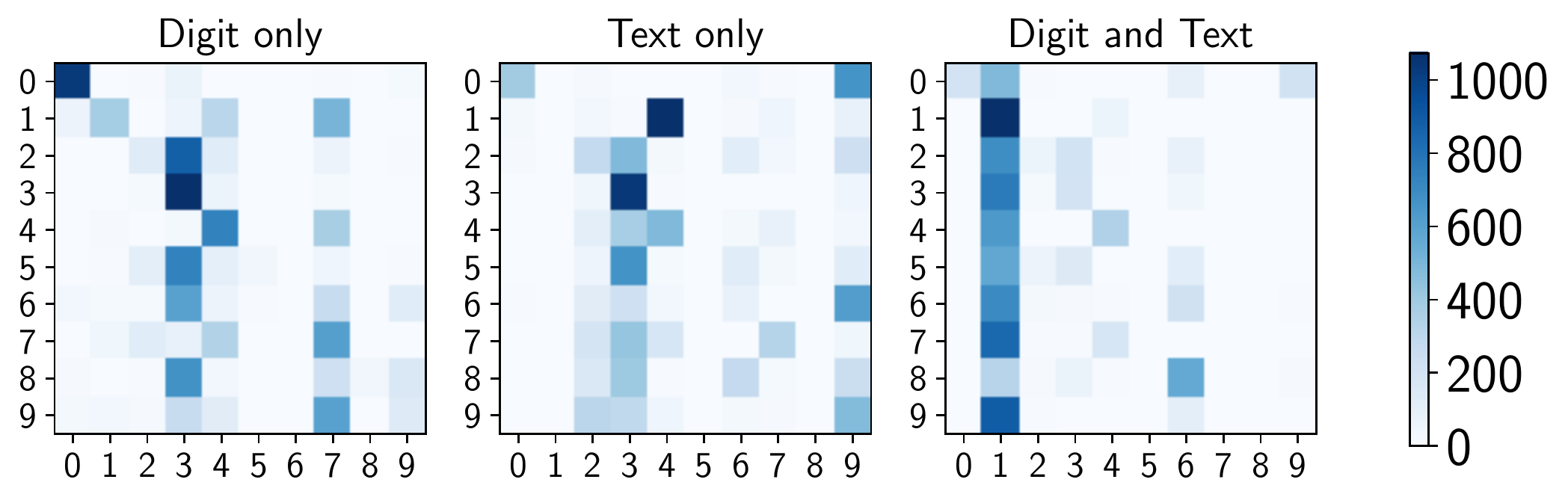}
\caption{\label{fig:mnist}Confusion matrices of BASIC-L's predictions on MNIST. \textbf{Digit only}: we use the class names \{``0'', ``1'', ..., ``9''\}; \textbf{Text only}: \{``one'', ``two'', ..., ``nine''\}; \textbf{Digit and Text}: \{``0 or zero'', ``1 or one'', ..., ``9 or nine''\}. The model has vastly different confusion matrices for different class name, suggesting that it does not understand the meaning of these strings, but instead, simply learns to match their embeddings.}
\end{figure}

\subsection{\label{sec:failed_cases}Example failure cases}
From the confusion matrices of BASIC-L on two benchmarks, EuroSAT~\citep{helber2018introducing} and MNIST~\citep{lecun2010mnist}, we observe that the prompts and class names are crucial for the performance of zero-shot transfor models. Here, we select and present a few examples to demonstrate the failures of BASIC-L. Figure~\ref{fig:falure_examples} visualizes these examples.

\FloatBarrier
\begin{center}
\begin{minipage}{0.125\linewidth}
    \includegraphics[width=1.0\linewidth, height=1.0\linewidth]{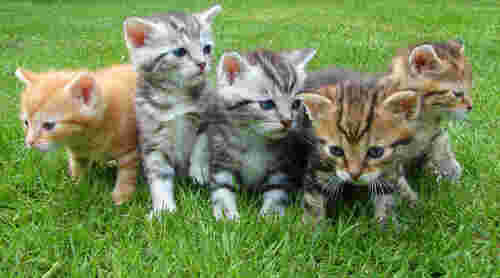}
\end{minipage}
\begin{minipage}{0.33\linewidth}
    \footnotesize
    \begin{tabularx}{\textwidth}{X|l}
    \emph{More than 6 kittens in total}. (\xmark) & $0.472$ \\
    \emph{More than 4 kittens in total}. & $0.342$ \\
    \emph{More than 2 kittens in total}. & $0.186$ \\
    \emph{More than 6 puppies in total}. & <1e-3 \\
    \emph{More than 2 puppies in total}. & <1e-3 \\
    \emph{More than 4 puppies in total}. & <1e-3 \\
    \end{tabularx}
\end{minipage}
\hfill
\begin{minipage}{0.125\linewidth}
    \includegraphics[width=1.0\linewidth, height=1.0\linewidth]{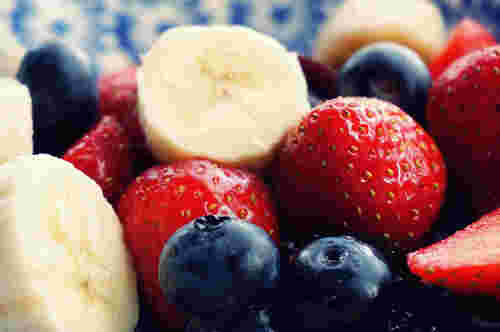}
\end{minipage}
\begin{minipage}{0.4\linewidth}
    \footnotesize
    \begin{tabularx}{\textwidth}{X|l}
    \emph{No strawberries found in the photo}. (\xmark) & $0.393$ \\
    \emph{No blueberries found in the photo}. & $0.304$ \\
    \emph{No bananas found in the photo}. &     $0.297$ \\
    \emph{No coconuts found in the photo}. & $0.003$ \\
    \emph{No pineapples found in the photo}. & $0.002$ \\
    \emph{No oranges found in the photo}. & $0.001$ \\
    \end{tabularx}
\end{minipage}

\vspace{0.005\linewidth}
\hrulefill
\vspace{0.005\linewidth}

\begin{minipage}{0.125\linewidth}
    \includegraphics[width=1.0\linewidth, height=1.0\linewidth]{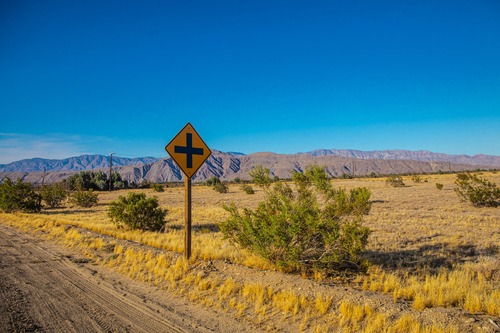}
\end{minipage}
\begin{minipage}{0.3\linewidth}
    \footnotesize
    \begin{tabularx}{\textwidth}{X|l}
    \emph{Closed road}. (\xmark) & $0.716$ \\
    \emph{Slippery road}. & $0.170$ \\
    \emph{Intersection}. & $0.076$ \\
    \emph{Stop}. & $0.034$ \\
    \emph{Sharp left}. & $0.003$ \\
    \emph{Sharp right}. & $0.002$ \\
    \end{tabularx}
\end{minipage}
\hfill
\begin{minipage}{0.125\linewidth}
    \includegraphics[width=1.0\linewidth, height=1.0\linewidth]{demos/traffic.jpeg}
\end{minipage}
\begin{minipage}{0.4\linewidth}
    \footnotesize
    \begin{tabularx}{\textwidth}{X|l}
    \emph{Traffic sign indicating intersection}. (\cmark) & $0.927$ \\
    \emph{Traffic sign indicating closed road}. & $0.027$ \\
    \emph{Traffic sign indicating sharp left}. & $0.027$ \\
    \emph{Traffic sign indicating sharp right}. & $0.016$ \\
    \emph{Traffic sign indicating slippery road}. & $0.003$ \\
    \emph{Traffic sign indicating stop}. & <1e-3 \\
    \end{tabularx}
\end{minipage}

\vspace{0.005\linewidth}
\hrulefill
\vspace{0.005\linewidth}

\begin{minipage}{0.125\linewidth}
    \includegraphics[width=1.0\linewidth, height=1.0\linewidth]{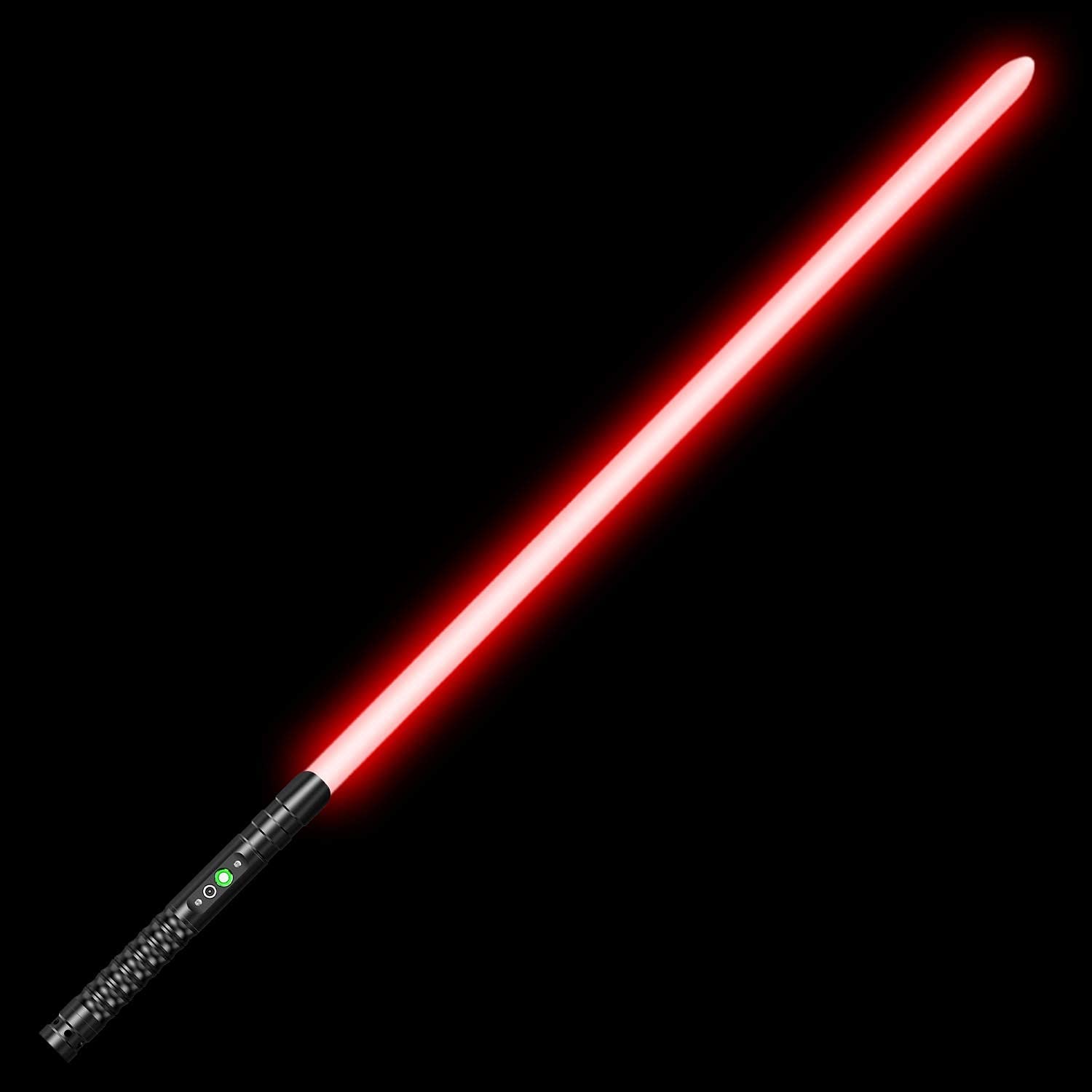}
\end{minipage}
\begin{minipage}{0.25\linewidth}
    \footnotesize
    \begin{tabularx}{\textwidth}{X|l}
    \emph{red light saber}. (\cmark) & 0.992 \\
    \emph{blue light saber}. & 0.005 \\
    \emph{red led light}. & 0.002 \\
    \emph{red neon light}. & <1e-3 \\
    \emph{blue led light}. & <1e-3 \\
    \emph{blue neon light}. & <1e-3 \\
    \end{tabularx}
\end{minipage}
\begin{minipage}{0.125\linewidth}
    \includegraphics[width=1.0\linewidth, height=1.0\linewidth]{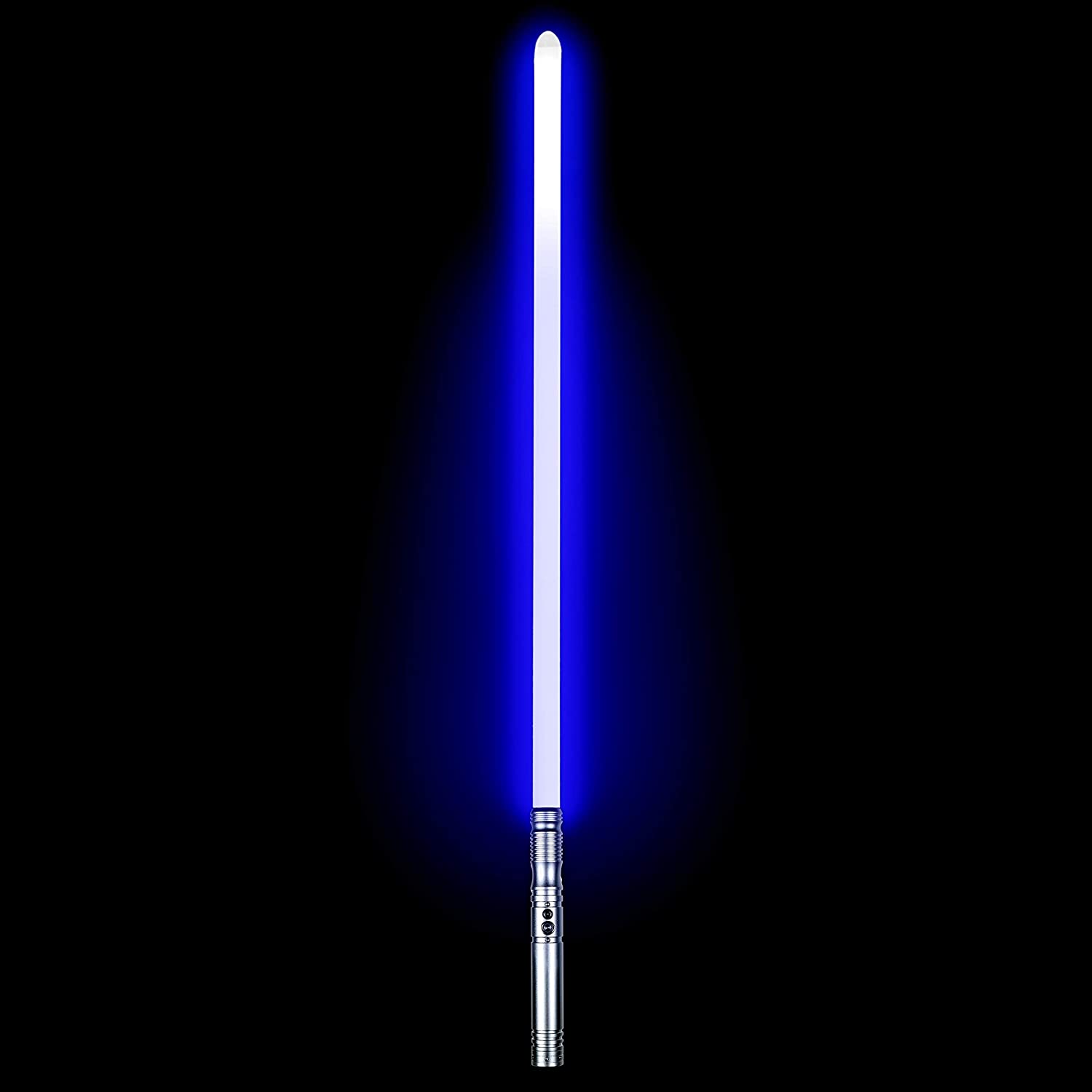}
\end{minipage}
\begin{minipage}{0.25\linewidth}
    \footnotesize
    \begin{tabularx}{\textwidth}{X|l}
    \emph{blue light saber}. (\cmark) & 0.990 \\
    \emph{red light saber}. & 0.004 \\
    \emph{blue led light}. & 0.003 \\
    \emph{blue neon light}. & 0.003 \\
    \emph{red led light}. & <1e-3 \\
    \emph{red neon light}. & <1e-3 \\
    \end{tabularx}
\end{minipage}

\begin{minipage}{0.125\linewidth}
    \includegraphics[width=1.0\linewidth, height=1.0\linewidth]{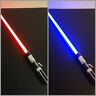}
\end{minipage}
\begin{minipage}{0.63\linewidth}
    \footnotesize
    \begin{tabularx}{\textwidth}{X|l}
    \emph{a blue light saber on the left and a red light saber on the right}. (\xmark) & $0.297$ \\
    \emph{a blue light saber on the right and a red light saber on the left}. & $0.235$ \\
    \emph{a red light saber on the left and a blue light saber on the right}. & $0.205$ \\
    \emph{a red light saber on the right and a blue light saber on the left}. & $0.173$ \\
    \emph{a red light saber to the right of a blue light saber}. & $0.054$ \\
    \emph{a red light saber to the left of a blue light saber}. & $0.0342$ \\
    \end{tabularx}
\end{minipage}

\captionof{figure}{\label{fig:falure_examples}Selected failure cases for BASIC-L over unseen images. \textbf{(1)} The first block indicates that the model is not precise in object counting and does not well handle negation in the prompts, possibly due the nature of our training data. \textbf{(2)} The middle block shows two examples to indicate that prompt engineering can play a critical role in providing the model with sufficient context to produce the desired output. \textbf{(3)} The last block shows that the model does not have the sense of left and right, which is a relic of random left-right flips of images which we apply during training.}
\end{center}

\let\FloatBarrier\relax

\section{Proofs} \label{app:proof}

In this appendix, we complete the proof of Theorem \ref{thm:1} and Theorem \ref{thm:2} by gradually analyzing the gap from the general case to the special case.

\subsection{General case}

\begin{lemma} \label{lemma:0}
Let $\Fcal_{}$  be a  set of maps $x\mapsto F(x)$ and $\Gcal$ be a set of maps $y\mapsto G(y)$.  Then, for any $\delta>0$, with probability at least $1-\delta$, the following holds for all $F \in \Fcal$ and $G \in \Gcal$:\begin{align}
&\EE_{x,y}[\bell_M(x, y)] -\hEE_S[\hell_B(x, y)] 
\\ \nonumber & \le c_1\EE_{x,y}[A(x,y)]\left(  \frac{1}{\sqrt{2B}} \left(2\sqrt{2}c_8 c_9+c_1 \sqrt{\kappa\ln( \sqrt{\kappa B}/\delta)} \right)+\frac{2}{c_1}\sup_{x\in\Xcal}\Rcal_{B}( \Hcal_{\Fcal,\Gcal,e}^{x}) \right)
\\ \nonumber & \quad +2 \Rcal_{m}( \Hcal_{\Fcal,\Gcal,\hell_B})+c_2 \sqrt{\frac{\ln(2/\delta)}{2m}}.
\end{align}
where  $\Hcal_{\Fcal,\Gcal,e}^x=\{y\mapsto \exp(F(x)\T G(y)):F \in \Fcal, G \in \Gcal\}$,  $\Hcal_{\Fcal,\Gcal,\hell_B}=\{(x,y)\mapsto \hell_B(x, y):F \in \Fcal, G \in \Gcal\}$, and 
$$
A(x,y)=\frac{\exp(F(x)\T G(y))}{\left(\frac{1}{B}\sum_{k=1}^B \exp(F(x)\T G(\hy_k))\right)\EE_{\by}[ \exp(F(x)\T G(\by))]}.
$$
Here,  $\Rcal_{m}(\Hcal):=\EE_{S,\sigma}[\sup_{h \in\Hcal}\frac{1}{m} \sum_{i=1}^m \sigma_i h(x_{i},y_{i})]$ where $\sigma_1,\dots,\sigma_m$ are independent uniform random variables taking values in $\{-1,1\}$.
\end{lemma}
\begin{proof} \ We first decompose the difference as
follows:\begin{align}
\EE_{x,y}[\bell_M(x, y)] -\hEE_S[\hell_B(x, y)]&=\EE_{x,y}[\bell_M(x, y)] -\EE_{x,y}[\hell_B(x, y)]+\EE_{x,y}[\hell_B(x, y)]-\hEE_S[\hell_B(x, y)]
\\ & =\EE_{x,y}[\bell_M(x, y)-\hell_B(x, y)]+(\EE_{x,y}[\hell_B(x, y)]-\hEE_S[\hell_B(x, y)]).  \label{eq:4}
\end{align}
For the inside of the expectation in the first term, we can write it as 
\begin{align}
&\bell_M(x, y)-\hell_B(x, y)
\\ &=  \frac{\exp(F(x)\T G(y))}{\frac{1}{B}\sum_{k=1}^B \exp(F(x)\T G(\hy_k))}-  \frac{\exp(F(x)\T G(y))}{\EE_{\by}[ \exp(F(x)\T G(\by))]}
\\ & =  \frac{\exp(F(x)\T G(y))\EE_{\by}[ \exp(F(x)\T G(\by))]-\exp(F(x)\T G(y))(\frac{1}{B}\sum_{k=1}^B \exp(F(x)\T G(\hy_k)))}{(\frac{1}{B}\sum_{k=1}^B \exp(F(x)\T G(\hy_k)))\EE_{\by}[ \exp(F(x)\T G(\by))]} 
\\ & =  \frac{\exp(F(x)\T G(y))\left(\EE_{\by}[ \exp(F(x)\T G(\by))]-\frac{1}{B}\sum_{k=1}^B \exp(F(x)\T G(\hy_k)) \right)}{\left(\frac{1}{B}\sum_{k=1}^B \exp(F(x)\T G(\hy_k))\right)\EE_{\by}[ \exp(F(x)\T G(\by))]} 
\label{eq:1}
\end{align}
Using Lemma \ref{lemma:1} (see below) with the assumption that $\hy_{1},\hy_{2},\dots,\hy_{B}\stackrel{iid}{\sim}p_{y} $ and  $\exp(F(x)\T G(y))\le c_1 $ with probability one, we have that  for any $\delta>0$ and $x\in \Xcal$, with probability at least $1-\delta$,
the following holds for all $F \in \Fcal$ and $G \in \Gcal$:
\begin{align} \label{eq:3}
\EE_{\by } [\exp(F(x)\T G(y))]- \frac{1}{B}\sum_{k=1}^B \exp(F(x)\T G(\hy_k))\le2\sup_{x\in\Xcal}\Rcal_{B}( \Hcal_{\Fcal,\Gcal,e}^{x})+c_1 \sqrt{\frac{\ln(1/\delta)}{2B}},
\end{align} 
where
$$
\sup_{x\in\Xcal}\Rcal_{B}( \Hcal_{\Fcal,\Gcal,e}^{x})=\sup_{x\in\Xcal}\EE_{y,\sigma}\left[\sup_{h_{x} \in\Hcal_{\Fcal,\Gcal,e}^x}\frac{1}{B} \sum_{i=1}^B \sigma_i h_{x}(y_{i})\right],
$$
and
$$
\Hcal_{\Fcal,\Gcal,e}^x=\{y\mapsto \exp(F(x)\T G(y)):F \in \Fcal, G \in \Gcal\}.
$$ 
Recall that $\gamma(x)=\EE_{\by } [\exp(F(x)\T G(y))]- \frac{1}{B}\sum_{k=1}^B \exp(F(x)\T G(\hy_k))$ is $c_{9}$-Lipschitz;  i.e., $|\gamma(x)-\gamma(x')|\le c_9 \|x-x'\|_2$ for all $x\in \Xcal$ where $\|x\|_2\le c_8$ for all $\Xcal$. For any metric space $(\Mcal,d)$ and subset $M\subseteq \Mcal$,  the (closed) ball of radius $r$ at centered at $c$ is denoted by $\Bcal_{(\Mcal,d)}[c,r]=\{x\in \Mcal : d(x,c)\le r\}$, and the  $r$-converging number of $M$\ with the covering $\Ccal$ is defined by 
$$
\Ncal_{\Ccal}(r,M)=\min\left\{|\Ccal|: \Ccal \subseteq \Mcal, M\subseteq \cup_{c \in\Ccal} \Bcal_{(\Mcal,d)}[c,r]\right.\}.
$$
Let us choose the metric space $(\Mcal,d)$ to be the Euclidian space $\RR^{\kappa}$ with the Euclidian metric. That is,  we have  the $\epsilon$-covering of $\Xcal$ with the Euclidean balls of radius $r$, with  the  $r$-converging number of 
$$
\Ncal_{\Ccal}(r,\Xcal) \le (2c_8 \sqrt{\kappa}/r)^\kappa. 
$$
Thus, by setting $r=\frac{2c_8}{\sqrt{B}}$, 
$$
\Ncal_{\Ccal}(r,\Xcal) \le ( \sqrt{\kappa B})^\kappa. 
$$
Using these, 
\begin{align*}
\sup_{x \in \Xcal} \gamma(x) = \inf_{c \in \Ccal} \sup_{x \in\Xcal}  \gamma(x) -\gamma(c) +\gamma(c)&\le  \inf_{c \in \Ccal} \sup_{x \in\Xcal}  |\gamma(x) -\gamma(c)| +\sup_{c \in \Ccal}\gamma(c) 
\\ & \le r c_{9}+\sup_{c \in \Ccal}\gamma(c)
\\ & =\frac{2c_8c_9}{\sqrt{B}}    +\sup_{c \in \Ccal}\gamma(c).
\end{align*}
Here, using \eqref{eq:3} with union bounds,  we have that  for any $\delta>0$, with probability at least $1-\delta$,
the following holds for all $F \in \Fcal$ and $G \in \Gcal$:
$$
\sup_{c \in \Ccal}\gamma(c) \le2\sup_{x\in\Xcal}\Rcal_{B}( \Hcal_{\Fcal,\Gcal,e}^{x})+c_1 \sqrt{\frac{\ln( \sqrt{\kappa B})^\kappa/\delta)}{2B}} \le2\sup_{x\in\Xcal}\Rcal_{B}( \Hcal_{\Fcal,\Gcal,e}^{x})+c_1 \sqrt{\frac{\kappa\ln( \sqrt{\kappa B}/\delta)}{2B}}.
$$
Therefore, for any $\delta>0$, with probability at least $1-\delta$, the following holds for all $F \in \Fcal$ and $G \in \Gcal$:
\begin{align}  \label{eq:7}
\sup_{x \in \Xcal} \gamma(x) &\le2\sup_{x\in\Xcal}\Rcal_{B}( \Hcal_{\Fcal,\Gcal,e}^{x})+\frac{2c_8}{\sqrt{B}} +c_1 \sqrt{\frac{\kappa\ln( \sqrt{\kappa B}/\delta)}{2B}}
\\ \nonumber & =2\sup_{x\in\Xcal}\Rcal_{B}( \Hcal_{\Fcal,\Gcal,e}^{x})+\frac{1}{\sqrt{2B}} \left(2\sqrt{2}c_8 c_9+c_1 \sqrt{\kappa\ln( \sqrt{\kappa B}/\delta)} \right).
\end{align}
Combining equations \eqref{eq:1} and \eqref{eq:7} with union bound, we have that for any $\delta>0$, with probability at least $1-\delta$,
\begin{align}
&\bell_M(x, y)-\hell_B(x, y)
\\ & =  \frac{\exp(F(x)\T G(y))\left(\EE_{\by} [\exp(F(x)\T G(y))]-\frac{1}{B}\sum_{k=1}^B \exp(F(x)\T G(\hy_k)) \right)}{\left(\frac{1}{B}\sum_{k=1}^B \exp(F(x)\T G(\hy_k))\right) \EE_{\by}[ \exp(F(x)\T G(\by))]}
\\ & \le  \frac{\exp(F(x)\T G(y))\left(\frac{1}{\sqrt{2B}} \left(2\sqrt{2}c_8 +c_1 \sqrt{\kappa\ln(c_9 \sqrt{\kappa B}/\delta)} \right)+2\sup_{x\in\Xcal}\Rcal_{B}( \Hcal_{\Fcal,\Gcal,e}^{x}) \right)}{\left(\frac{1}{B}\sum_{k=1}^B \exp(F(x)\T G(\hy_k))\right)\EE_{\by}[ \exp(F(x)\T G(\by))]}.
\end{align}
By defining 
$$
A(x,y)=\frac{\exp(F(x)\T G(y))}{\left(\frac{1}{B}\sum_{k=1}^B \exp(F(x)\T G(\hy_k))\right)\EE_{\by}[ \exp(F(x)\T G(\by))]},
$$
we have that for any $\delta>0$, with probability at least $1-\delta$,
\begin{align} \label{eq:5}
&\EE_{x,y}[\bell_M(x, y)-\hell_B(x, y)]
\\ \nonumber & \le c_1\EE_{x,y}[A(x,y)]\left(  \frac{1}{\sqrt{2B}} \left(2\sqrt{2}c_8  c_9+c_1 \sqrt{\kappa\ln( \sqrt{\kappa B}/\delta)} \right)+\frac{2}{c_1}\sup_{x\in\Xcal}\Rcal_{B}( \Hcal_{\Fcal,\Gcal,e}^{x}) \right).
\end{align}
For the second term, using Lemma \ref{lemma:1}  with the assumption that $\hell_B(x, y) \le c_2$ for $(x, y) \sim p_{(x,y)}$,  we have that  for any $\delta>0$, with probability at least $1-\delta$,
the following holds for all $F \in \Fcal$ and $G \in \Gcal$:
\begin{align} \label{eq:6}
&\EE_{x,y}[\hell_B(x, y)]-\hEE_S[\hell_B(x, y)]
\\ &=\EE_{x,y}[\hell_B(x, y)]-\frac{1}{m}\sum_{i=1}^m-  \frac{B\exp(F(\hx_{i})\T G(\hy_{i}))}{\sum_{k=1}^B \exp(F(\hx_i)\T G(\hy_k))}\le 2 \Rcal_{m}( \Hcal_{\Fcal,\Gcal,\hell_B})+c_2 \sqrt{\frac{\ln(1/\delta)}{2m}}, 
\end{align}
where  $\Hcal_{\Fcal,\Gcal,\hell_B}=\{(x,y)\mapsto \hell_B(x, y):F \in \Fcal, G \in \Gcal\}$. 
Combining equations \eqref{eq:4}, \eqref{eq:5}, and \eqref{eq:6} with union bound, we have that for any $\delta>0$, with probability at least $1-\delta$,
the following holds  all $F \in \Fcal$ and $G \in \Gcal$:
\begin{align*}
&\EE_{x,y}[\bell_M(x, y)] -\hEE_S[\hell_B(x, y)] 
\\ &\ \le c_1\EE_{x,y}[A(x,y)]\left(  \frac{1}{\sqrt{2B}} \left(2\sqrt{2}c_8 c_9+c_1 \sqrt{\kappa\ln( \sqrt{\kappa B}/\delta)} \right)+\frac{2}{c_1}\sup_{x\in\Xcal}\Rcal_{B}( \Hcal_{\Fcal,\Gcal,e}^{x}) \right) 
\\ & \quad +2\Rcal_{m}( \Hcal_{\Fcal,\Gcal,\hell_B})+c_2 \sqrt{\frac{\ln(2/\delta)}{2m}}. 
\end{align*}

\end{proof}

The proof of Lemma \ref{lemma:0} partially builds up on Lemma \ref{lemma:1} below. Lemma \ref{lemma:1} is a direct application of previous results \citep{bartlett2002rademacher,mohri2012foundations,shalev2014understanding} to our problem. We provide a proof of Lemma \ref{lemma:1} by slightly modifying the proof of a previous work \citep[Theorem 3.1]{mohri2012foundations} for the completeness (the proof utilizes the nonnegativity of $h$ to have a slightly tighter bound than Theorem 26.5 of \citealp{shalev2014understanding}): 

\begin{lemma} \label{lemma:1}
Let $\Hcal$  be a  set of maps $z\mapsto h(z)$ such that $h(z) \in [0, \lambda]$ for all $z$ in its domain. Then, for any $\delta>0$, with probability at least $1-\delta$ over  an i.i.d. draw of $m$ i.i.d.  samples  $(z_{i})_{i=1}^m$, the following holds for all maps $ h\in\Hcal$:
\begin{align} \label{eq:new:1}
\EE_{z}[h(z)]
\le \frac{1}{m}\sum_{i=1}^{m} h(z_{i})+2 \Rcal_{m}(\Hcal)+\lambda \sqrt{\frac{\ln(1/\delta)}{2m}},   
\end{align}
where  $\Rcal_{m}(\Hcal):=\EE_{(z_{1},\dots,z_m),\sigma}[\sup_{h \in\Hcal}\frac{1}{m} \sum_{i=1}^m \sigma_i h(z_{i})]$ where $\sigma_1,\dots,\sigma_m$ are independent uniform random variables taking values in $\{-1,1\}$.  
\end{lemma}
\begin{proof} \
 Let $S=(z_i)_{i=1}^m$ and $S'=(z_i')_{i=1}^m$. Define 
\begin{align}
\varphi(S)= \sup_{h \in\Hcal} \EE_{x,y}[h(z)]-\frac{1}{m}\sum_{i=1}^{m}h(z_i).
\end{align} 
To apply McDiarmid's inequality to $\varphi(S)$, we compute an upper bound on $|\varphi(S)-\varphi(S')|$ where  $S$ and $S'$ be two test datasets differing by exactly one point of an arbitrary index $i_{0}$; i.e.,  $S_i= S'_i$ for all $i\neq i_{0}$ and $S_{i_{0}} \neq S'_{i_{0}}$. Then,
\begin{align}
\varphi(S')-\varphi(S) \le\sup_{h \in\Hcal}\frac{h(z_{i_{0}})-h(z_{i_{0}}')}{m} \le \frac{\lambda}{m}.
\end{align}
Thus, by McDiarmid's inequality, for any $\delta>0$, with probability at least $1-\delta$,
\begin{align}
\varphi(S) \le  \EE_{S}[\varphi(S)] + \lambda \sqrt{\frac{\ln(1/\delta)}{2m}}.
\end{align}
Moreover, \begin{align}
&\EE_{S}[\varphi(S)] 
\\ &  = \EE_{S}\left[\sup_{h \in\Hcal} \EE_{S'}\left[\frac{1}{m}\sum_{i=1}^{m}h(z_i')\right]-\frac{1}{m}\sum_{i=1}^{m}h(z_i)\right]   
 \\ &  \le\EE_{S,S'}\left[\sup_{h \in\Hcal} \frac{1}{m}\sum_{i=1}^m (h(z_i')-h(z_i)\right] 
 \\ & \le \EE_{\xi, S, S'}\left[\sup_{h\in\Hcal} \frac{1}{m}\sum_{i=1}^m  \xi_i(h(z_i')-h(z_i))\right]
 \\ &  \le2\EE_{\xi, S}\left[\sup_{h\in\Hcal} \frac{1}{m}\sum_{i=1}^m  \xi_ih(z_i))\right] =2\Rcal_{m}( \Hcal)  
\end{align}
where  the fist line follows the definitions of each term, the second line uses the Jensen's inequality and the convexity of  the 
supremum, and the third line follows that for each $\xi_i \in \{-1,+1\}$, the distribution of each term $\xi_i (h(z_i')-h(z_i))$ is the  distribution of  $(h(z_i')-h(z_i))$  since $S$ and $S'$ are drawn iid with the same distribution. The forth line uses the subadditivity of supremum.

\end{proof}  

\subsubsection{Analyzing $\sup_{x\in\Xcal}\Rcal_{B}( \Hcal_{\Fcal,\Gcal,e}^{x})$ and $\Rcal_{m}( \Hcal_{\Fcal,\Gcal,\hell_B})$}

The bound in Lemma \ref{lemma:0} contains two complex terms, $\sup_{x\in\Xcal}\Rcal_{B}( \Hcal_{\Fcal,\Gcal,e}^{x})$ and $\Rcal_{m}( \Hcal_{\Fcal,\Gcal,\hell_B})$, that are challenging to interpret and to be further analyzed. The following two lemmas bound those two terms by more interpretable quantities:

\begin{lemma} \label{lemma:4}
Let $\Fcal_{}$  be a  set of maps $x\mapsto F(x)$ and $\Gcal$ be a set of maps $y\mapsto G(y)$. Then, 
$$
\sup_{x\in\Xcal}\Rcal_{B}( \Hcal_{\Fcal,\Gcal,e}^{x})\le c_1c_3\EE_{y,\sigma}\left[\sup_{ G \in \Gcal} \frac{1}{B}  \left\|\sum_{i=1}^B \sigma_i  G(y_{i})\right\|_2\right].
$$
\end{lemma}
\begin{proof} \ 
Since the derivative of exponential function $\exp(q)$ is $\exp(q)$ and we assume $\exp(F(x)\T G(y))\le c_1$, the exponential function in the bounded domain of $\exp(F(x)\T G(y))\le c_1$ has Lipschitz constant of $c_1$. Therefore,
\begin{align*}
\EE_{y,\sigma}\left[\sup_{h_{x} \in\Hcal_{\Fcal,\Gcal,e}^x}\frac{1}{B} \sum_{i=1}^B \sigma_i h_{x}(y_{i})\right] &=\EE_{y,\sigma}\left[\sup_{F \in \Fcal, G \in \Gcal}\frac{1}{B} \sum_{i=1}^B \sigma_i \exp(F(x)\T G(y_{i}))\right]  
\\ & \le c_1\EE_{y,\sigma}\left[\sup_{F \in \Fcal, G \in \Gcal}\frac{1}{B} \sum_{i=1}^B \sigma_i F(x)\T G(y_{i})\right] 
\\ & =  \frac{c_1}{B}\EE_{y,\sigma}\left[\sup_{F \in \Fcal, G \in \Gcal}F(x)\T \sum_{i=1}^B \sigma_i  G(y_{i})\right].
\\ & \le   \frac{c_1}{B}\EE_{y,\sigma}\left[\sup_{F \in \Fcal, G \in \Gcal} \|F(x)\|_2 \left\|\sum_{i=1}^B \sigma_i  G(y_{i})\right\|_2\right] 
\\ & \le   \frac{c_1}{B}\left(\sup_{F \in \Fcal}\|F(x)\|_2\right)\EE_{y,\sigma}\left[\sup_{ G \in \Gcal}  \left\|\sum_{i=1}^B \sigma_i  G(y_{i})\right\|_2\right] 
\\ & =c_1\left(\sup_{F \in \Fcal}\|F(x)\|_2\right)\EE_{y,\sigma}\left[\sup_{ G \in \Gcal}  \frac{1}{B}\left\|\sum_{i=1}^B \sigma_i  G(y_{i})\right\|_2\right]. \end{align*}
Therefore,
\begin{align*}
\sup_{x\in\Xcal}\Rcal_{B}( \Hcal_{\Fcal,\Gcal,e}^{x})&=\sup_{x\in\Xcal}\EE_{y,\sigma}\left[\sup_{h_{x} \in\Hcal_{\Fcal,\Gcal,e}^x}\frac{1}{B} \sum_{i=1}^B \sigma_i h_{x}(y_{i})\right]
\\ & \le c_1\left(\sup_{x\in\Xcal ,F \in \Fcal}\|F(x)\|_2\right)\EE_{y,\sigma}\left[\sup_{ G \in \Gcal} \frac{1}{B}  \left\|\sum_{i=1}^B \sigma_i  G(y_{i})\right\|_2\right].
\end{align*}
\end{proof}
\begin{lemma} \label{lemma:5}
Let $\Fcal_{}$  be a  set of maps $x\mapsto F(x)$ and $\Gcal$ be a set of maps $y\mapsto G(y)$. Then, 
$$
\Rcal_{m}( \Hcal_{\Fcal,\Gcal,\hell_B})\le\sqrt{2} c_4 \sqrt{c_5^{2}+c_6^{2}} \sum_{k=1}^{D}\left(\Rcal_ m( \Fcal_{k})+\Rcal_ m( \Gcal_{k})   \right).
$$
where $\Fcal_k=\{x \mapsto F(x)_k : F \in \Fcal\}$ and $\Gcal_k =\{y \mapsto G(y)_{k} : G \in \Gcal\}$.
\end{lemma}
\begin{proof} \ 
Recall that $$\Hcal_{\Fcal,\Gcal,\hell_B}=\{(x,y)\mapsto \hell_B(x, y):F \in \Fcal, G \in \Gcal\}
$$
Using the definitions,\begin{align*}
\Rcal_{m}( \Hcal_{\Fcal,\Gcal,\hell_B}) &=\EE_{(x,y),\sigma}\left[\sup_{h \in \Hcal_{\Fcal,\Gcal,\hell_B}}\frac{1}{m} \sum_{i=1}^m \sigma_i h(x_{i},y_{i})\right]
\\ & =\EE_{(x,y),\sigma}\left[\sup_{F \in \Fcal, G \in \Gcal}\frac{1}{m} \sum_{i=1}^m \sigma_i  \hell_B(x_{i}, y_{i})\right] 
\\ &=\frac{1}{m}\EE_{(x,y),\sigma}\left[\sup_{F \in \Fcal, G \in \Gcal} \sum_{i=1}^m \sigma_i   \frac{B\exp(F(x_{i})\T G(y_{i}))}{\sum_{k=1}^B \exp(F(x_{i})\T G(\hy_k))}\right].  
\end{align*}
Define 
$$
h(p,q)=\frac{B\exp(p\T q)}{\sum_{k=1}^B \exp(p\T G(\hy_k))}.
$$
Then, 
$$
\Rcal_{m}( \Hcal_{\Fcal,\Gcal,\hell_B}) =\frac{1}{m}\EE_{(x,y),\sigma}\left[\sup_{F \in \Fcal, G \in \Gcal} \sum_{i=1}^m \sigma_i   h(F(x_{i}),G(y_{i}))\right].
$$
Moreover, 
\begin{align*}
\frac{\partial h(p,q)}{\partial p} =\frac{B\exp(p\T q)}{\sum_{k=1}^B \exp(p\T G(\hy_k))} q\T -\frac{B\exp(p\T q)}{\left(\sum_{k=1}^B \exp(p\T G(\hy_k))\right)^2} \left(\sum_{k=1}^B \exp(p\T G(\hy_k))G(\hy_k)\T \right)
\end{align*}
\begin{align*}
\frac{\partial h(p,q)}{\partial q} =\frac{B\exp(p\T q)}{\sum_{k=1}^B \exp(p\T G(\hy_k))} p\T. 
\end{align*}
Therefore,
\begin{align*}
\|\nabla h(p,q)\|^2_2 &=\left(\frac{\exp(p\T q)}{\frac{1}{B}\sum_{k=1}^B \exp(p\T G(\hy_k))}\right)^2 \left[\sum_{i=1}^D \left(q_i - \frac{\sum_{k=1}^B \exp(p\T G(\hy_k))G(\hy_k)_i}{\sum_{k=1}^B \exp(p\T G(\hy_k))}\right)^2 +\sum_{i=1}^D p_i^2 \right]
\\ &\le c_4 ^{2}(c_5^{2}+c_6^{2}).
\end{align*}
Thus, 
$$
\|\nabla h(p,q)\|_2 \le c_4 \sqrt{c_5^{2}+c_6^{2}}.
$$
Using a vector-contraction inequality, i.e., Corollary 4 of \citep{maurer2016vector} with the additional expectation of both sides of the inequality, we have that 
\begin{align*}
&\Rcal_{m}( \Hcal_{\Fcal,\Gcal,\hell_B}) 
\\ &=\frac{1}{m}\EE_{(x,y),\sigma}\left[\sup_{F \in \Fcal, G \in \Gcal} \sum_{i=1}^m \sigma_i   h(F(x_{i}),G(y_{i}))\right]
\\ & \le \frac{\sqrt{2} c_4 \sqrt{c_5^{2}+c_6^{2}}}{m}\EE_{(x,y),\sigma}\left[\sup_{F \in \Fcal, G \in \Gcal} \sum_{i=1}^m \sum_{k=1}^D \sigma_{ik}  F(x_{i})_k+\sum_{i=1}^m \sum_{j=1}^{D}\sigma_{ij}G(y_{i})_{j}\right]
\\ & \le\frac{\sqrt{2} c_4 \sqrt{c_5^{2}+c_6^{2}}}{m}\EE_{(x,y),\sigma}\left[\sup_{F \in \Fcal} \sum_{i=1}^m \sum_{k=1}^D \sigma_{ik}  F(x_{i})_k+\sup_{G \in \Gcal}\sum_{i=1}^m \sum_{j=1}^{D}\sigma_{ij}G(y_{i})_{j}\right]
\\ & =\frac{\sqrt{2} c_4 \sqrt{c_5^{2}+c_6^{2}}}{m}\left(\EE_{(x,y),\sigma}\left[\sup_{F \in \Fcal} \sum_{i=1}^m \sum_{k=1}^D \sigma_{ik}  F(x_{i})_k\right]+\EE_{(x,y),\sigma}\left[\sup_{G \in \Gcal}\sum_{i=1}^m \sum_{j=1}^{D}\sigma_{ij}G(y_{i})_{j}\right]   \right)
\\ & \le\frac{\sqrt{2}c_4 \sqrt{c_5^{2}+c_6^{2}}}{m}\left(\sum_{k=1}^D\EE_{x,\sigma}\left[\sup_{f\in \Fcal_{k}} \sum_{i=1}^m  \sigma_{i}  f(x_{i})\right]+\sum_{k=1}^{D}\EE_{y,\sigma}\left[\sup_{g \in \Gcal_{k}}\sum_{i=1}^m \sigma_{i}g(y_{i})\right]   \right) 
\\ & =\sqrt{2}c_4 \sqrt{c_5^{2}+c_6^{2}} \left(\sum_{k=1}^D\Rcal_ m( \Fcal_{k})+\sum_{k=1}^{D}\Rcal_ m( \Gcal_{k})   \right) 
\end{align*}
\end{proof}

\subsubsection{Combining all together for the general case}
We now combine the above lemmas to complete the proof of Theorem \ref{thm:2}:
%\thmb*
\begin{proof}[Proof of Theorem \ref{thm:2}]
From Lemma \ref{lemma:0}, we have that for any  for any $\delta>0$, with probability at least $1-\delta$, the following holds for all $F \in \Fcal$ and $G \in \Gcal$:
\begin{align*}
&\EE_{x,y}[\bell_M(x, y)] -\hEE_S[\hell_B(x, y)] 
\\ \nonumber & \le c_1\EE_{x,y}[A(x,y)]\left(  \frac{1}{\sqrt{2B}} \left(2\sqrt{2}c_8 c_9+c_1 \sqrt{\kappa\ln( \sqrt{\kappa B}/\delta)} \right)+\frac{2}{c_1}\sup_{x\in\Xcal}\Rcal_{B}( \Hcal_{\Fcal,\Gcal,e}^{x}) \right)
\\ & \quad +2 \Rcal_{m}( \Hcal_{\Fcal,\Gcal,\hell_B})+c_2 \sqrt{\frac{\ln(2/\delta)}{2m}}.
\end{align*}
Then by using Lemma \ref{lemma:4} and \ref{lemma:5}, 
\begin{align*}
&\EE_{x,y}[\bell_M(x, y)] -\hEE_S[\hell_B(x, y)] 
\\ \nonumber & \le c_1\EE_{x,y}[A(x,y)]\left(  \frac{1}{\sqrt{2B}} \left(2\sqrt{2}c_8 c_9+c_1 \sqrt{\kappa\ln( \sqrt{\kappa B}/\delta)} \right)+2c_3\EE_{y,\sigma}\left[\sup_{ G \in \Gcal} \frac{1}{B}  \left\|\sum_{i=1}^B \sigma_i  G(y_{i})\right\|_2\right] \right)
\\ & \qquad +2\sqrt{2} c_4 \sqrt{c_5^{2}+c_6^{2}} \sum_{k=1}^{D}\left(\Rcal_ m( \Fcal_{k})+\Rcal_ m( \Gcal_{k})   \right)+c_2 \sqrt{\frac{\ln(2/\delta)}{2m}}
\\ & =  \frac{ C_1 }{\sqrt{2B}}+C_{2}\Rcal_B(G)  + C_{3} \sum_{k=1}^{D}\left(\Rcal_ m( \Fcal_{k})+\Rcal_ m( \Gcal_{k})   \right)+c_2 \sqrt{\frac{\ln(2/\delta)}{2m}}.
\end{align*}
\end{proof}

\subsection{Bounding $\EE_{y,\sigma}\left[\sup_{ G \in \Gcal} \frac{1}{B}  \left\|\sum_{i=1}^B \sigma_i  G(y_{i})\right\|_2\right]$ and $\sum_{k=1}^{D}\left(\Rcal_ m( \Fcal_{k})+\Rcal_ m( \Gcal_{k})   \right)$ for the special case with deep neural networks}

We now want to bound $\EE_{y,\sigma}\left[\sup_{ G \in \Gcal}  \left\|\sum_{i=1}^m \sigma_i  G(y_{i})\right\|_2\right]$ and $\sum_{k=1}^{D}\left(\Rcal_ m( \Fcal_{k})+\Rcal_ m( \Gcal_{k})   \right)$ in the case where $F$ and $G$ represent deep neural networks. We consider standard deep neural networks, of the form
$$
G(y)= (\omega_{L} \circ \sigma_{L-1}\circ \omega_{L-1}\circ \sigma_{L-2}\cdots \sigma_1\circ \omega_1)(y)
$$
$$
F(x)= (\omega_{L'} '\circ \sigma_{L'-1} '\circ \omega_{L'-1} '\circ \sigma_{L'-2}' \cdots \sigma_1 '\circ \omega_1')(x)
$$
where $\omega_l(q)=W_{l}q$ and $\sigma_{l}$ is an element-wise activation function. Similarly,  $\omega_l'(q)=W_{l}'q$ and $\sigma_{l}'$ is an element-wise activation function.
\begin{lemma} \label{lemma:2}
Suppose that the function  $\sigma_{l}$ is 1-Lipschitz and positive homogeneous for all $l \in [L-1]$ and $\|y\|_{2}\le c_7$ for all $y \in \Ycal$. Let $ \Gcal=\{y  \mapsto G(y): (\forall l \in[L])[\|W_{l}^{}\|_F \le M_l] \}$. Then,  
$$
\EE_{y,\sigma}\left[\sup_{ G \in \Gcal} \frac{1}{B}  \left\|\sum_{i=1}^B \sigma_i  G(y_{i})\right\|_2\right]\ \le\frac{c_7 (\sqrt{2 \log(2) L }+1)(\prod_{l=1}^L M_l)}{\sqrt{B}}.
$$
\end{lemma}
\begin{proof} \ 
Since 
\begin{align*}
\left\|\sum_{i=1}^B \sigma_i  G(y_{i})\right\|_2 &=\left\|\sum_{i=1}^B \sigma_i  W_{L}(\sigma_{L-1}\circ \omega_{L-1}\circ \sigma_{L-2}\cdots \sigma_1\circ \omega_1)(y)\right\|_2 
\\ & \le \|W_L \|_F \left\|\sum_{i=1}^B \sigma_i  (\sigma_{L-1}\circ \omega_{L-1}\circ \sigma_{L-2}\cdots \sigma_1\circ \omega_1)(y)\right\|_2, 
\end{align*}
the  proof steps of Theorem 1 of \citep{golowich2018size} work to bound
$\EE_{\sigma}\left[\sup_{ G \in \Gcal}  \left\|\sum_{i=1}^B \sigma_i  G(y_{i})\right\|_2\right]$. 
Therefore, using the proof of Theorem 1 of \citep{golowich2018size},
$$
\EE_{\sigma}\left[\sup_{ G \in \Gcal}  \frac{1}{B}\left\|\sum_{i=1}^B \sigma_i  G(y_{i})\right\|_2\right]\le \frac{c_7 (\sqrt{2 \log(2) L }+1)(\prod_{l=1}^L M_l)}{\sqrt{B}}  $$
\end{proof}
\begin{lemma}\label{lemma:3} 
Suppose that the function  $\sigma_{l}'$ is 1-Lipschitz and positive homogeneous for all $l \in [L-1]$ and $\|x\|\le c_8$ for all $x \in \Xcal$. Let $ \Fcal=\{x  \mapsto F(x): (\forall l \in[L'-1])[\|W_{l}'\|_F \le M_l'] \wedge  \|(W_{l}')_{k}\|_F \le M_{L',k}'\}$ where $(W_{l}')_{k}$ is the $k$-th row  of $W_{l}'$. Suppose that the function  $\sigma_{l}$ is 1-Lipschitz and positive homogeneous for all $l \in [L-1]$ and $\|y\|\le c_7$ for all $y \in \Ycal$. Let $ \Gcal=\{y  \mapsto G(y): (\forall l \in[L-1])[\|W_{l}\|_F \le M_l] \wedge  \|(W_{L})_{k}\|_F \le M_{L,k}\}$ where $(W_{l})_{k}$ is the $k$-th row  of $W_{l}$. Then, 
\begin{align*}
& \sum_{k=1}^{D}\left(\Rcal_ m( \Fcal_{k})+\Rcal_ m( \Gcal_{k})   \right) 
\\ & \le \frac{c_7 (\sqrt{2 \log(2) L }+1)(\prod_{l=1}^{L-1} M_l)\sum_{k=1}^D M_{L,k}}{\sqrt{m}}+\frac{c_8 (\sqrt{2 \log(2) L' }+1)(\prod_{l=1}^{L'-1} M_l')\sum_{k=1}^D M'_{L',k}}{\sqrt{m}} 
\end{align*}  
and
$$
\EE_{y,\sigma}\left[\sup_{ G \in \Gcal} \frac{1}{B}  \left\|\sum_{i=1}^B \sigma_i  G(y_{i})\right\|_2\right] \le\frac{c_7 (\sqrt{2 \log(2) L }+1)(\prod_{l=1}^{L-1} M_l)\sqrt{\sum_{k=1}^D M_{L,k}^2}}{\sqrt{B}}.
$$

\end{lemma}
\begin{proof} \ 
From Theorem 1 of \citep{golowich2018size}, we have that~$$
\Rcal_ m( \Fcal_{k})\le\frac{c_8 (\sqrt{2 \log(2) L' }+1)(\prod_{l=1}^{L'-1} M_l')M'_{L',k}}{\sqrt{m}}
$$
and
$$
\Rcal_ m( \Gcal_{k})\le\frac{c_7 (\sqrt{2 \log(2) L }+1)(\prod_{l=1}^{L-1} M_l)M_{L,k}}{\sqrt{m}}.
$$ 
Thus, 
\begin{align*}
& \sum_{k=1}^{D}\left(\Rcal_ m( \Fcal_{k})+\Rcal_ m( \Gcal_{k})   \right) 
\\ & \le\sum_{k=1}^D \left( \frac{c_7 (\sqrt{2 \log(2) L }+1)(\prod_{l=1}^{L-1} M_l)M_{L,k}}{\sqrt{m}}+\frac{c_8 (\sqrt{2 \log(2) L' }+1)(\prod_{l=1}^{L'-1} M_l') M'_{L',k}}{\sqrt{m}}\right). 
\end{align*} 
This proves the first statement. For the second statement, since $\|(W_{l})_{k}\|_F \le M_{L,k}$, we have that  
$$
\|W_{L}^{}\|_F^2 = \sum_{k=1}^D\|(W_{L})_{k}\|_F^2 \le \sum_{k=1}^D M_{L,k}^{2}. $$ 
This implies that $\|W_{l}^{}\|_F\le\sqrt{\sum_{k=1}^D M_{L,k}^2}$. Thus, using Lemma \ref{lemma:2}, 
$$
\EE_{y,\sigma}\left[\sup_{ G \in \Gcal} \frac{1}{B}  \left\|\sum_{i=1}^B \sigma_i  G(y_{i})\right\|_2\right]\ \le\frac{c_7 (\sqrt{2 \log(2) L }+1)(\prod_{l=1}^{L-1} M_l)\sqrt{\sum_{k=1}^D M_{L,k}^2}}{\sqrt{B}}.
$$
\end{proof}

\subsection{Combining all together for the special case with deep neural networks}
We now combine the above lemmas to complete the proof of Theorem \ref{thm:1} for the special case with deep neural networks:
%\thma*
\begin{proof}[Proof of Theorem \ref{thm:1}] \ 
From Lemma \ref{lemma:0},  for any $\Fcal$ and $\Gcal$, and for any $\delta>0$, with probability at least $1-\delta$, the following holds for all $F \in \Fcal$ and $G \in \Gcal$:
\begin{align*}
&\EE_{x,y}[\bell_M(x, y)] -\hEE_S[\hell_B(x, y)] 
\\ \nonumber & \le c_1\EE_{x,y}[A(x,y)]\left(  \frac{1}{\sqrt{2B}} \left(2\sqrt{2}c_8 c_9+c_1 \sqrt{\kappa\ln( \sqrt{\kappa B}/\delta)} \right)+\frac{2}{c_1}\sup_{x\in\Xcal}\Rcal_{B}( \Hcal_{\Fcal,\Gcal,e}^{x}) \right)
\\ & \quad +2 \Rcal_{m}( \Hcal_{\Fcal,\Gcal,\hell_B})+c_2 \sqrt{\frac{\ln(2/\delta)}{2m}},
\end{align*}
where Lemma \ref{lemma:4} and Lemma \ref{lemma:5} show that 
$$
\sup_{x\in\Xcal}\Rcal_{B}( \Hcal_{\Fcal,\Gcal,e}^{x})\le c_1c_3\EE_{y,\sigma}\left[\sup_{ G \in \Gcal} \frac{1}{B}  \left\|\sum_{i=1}^B \sigma_i  G(y_{i})\right\|_2\right],
$$
and 
$$
\Rcal_{m}( \Hcal_{\Fcal,\Gcal,\hell_B})\le\sqrt{2} c_4 \sqrt{c_5^{2}+c_6^{2}} \sum_{k=1}^{D}\left(\Rcal_ m( \Fcal_{k})+\Rcal_ m( \Gcal_{k})   \right).
$$
Finally, using Lemma \ref{lemma:3} for the particular  $\Fcal$ and $\Gcal$ with deep neural networks, we have that 
\begin{align*}
& \sum_{k=1}^{D}\left(\Rcal_ m( \Fcal_{k})+\Rcal_ m( \Gcal_{k})   \right) 
\\ & \le \frac{c_7 (\sqrt{2 \log(2) L }+1)(\prod_{l=1}^{L-1} M_l)\sum_{k=1}^D M_{L,k}}{\sqrt{m}}+\frac{c_8 (\sqrt{2 \log(2) L' }+1)(\prod_{l=1}^{L'-1} M_l')\sum_{k=1}^D M'_{L',k}}{\sqrt{m}} 
\end{align*}  
and
$$
\EE_{y,\sigma}\left[\sup_{ G \in \Gcal} \frac{1}{B}  \left\|\sum_{i=1}^B \sigma_i  G(y_{i})\right\|_2\right] \le\frac{c_7 (\sqrt{2 \log(2) L }+1)(\prod_{l=1}^{L-1} M_l)\sqrt{\sum_{k=1}^D M_{L,k}^2}}{\sqrt{B}}.
$$
Combining those, we have that  for any $\delta>0$, with probability at least $1-\delta$, the following holds for all $F \in \Fcal$ and $G \in \Gcal$:
\begin{align*}
&\EE_{x,y}[\bell_M(x, y)] -\hEE_S[\hell_B(x, y)] 
\\ \nonumber & \le c_1\EE_{x,y}[A(x,y)]\left(  \frac{1}{\sqrt{2B}} \left(2\sqrt{2}c_8 c_9+c_1 \sqrt{\kappa\ln( \sqrt{\kappa B}/\delta)} \right)+2c_3\EE_{y,\sigma}\left[\sup_{ G \in \Gcal} \frac{1}{B}  \left\|\sum_{i=1}^B \sigma_i  G(y_{i})\right\|_2\right] \right)
\\ & \qquad +2 \sqrt{2} c_4 \sqrt{c_5^{2}+c_6^{2}} \sum_{k=1}^{D}\left(\Rcal_ m( \Fcal_{k})+\Rcal_ m( \Gcal_{k})   \right)+c_2 \sqrt{\frac{\ln(2/\delta)}{2m}}
\\ \nonumber & \le c_1\EE_{x,y}[A(x,y)]  \frac{1}{\sqrt{2B}} \left(2\sqrt{2}c_8 c_9+c_1 \sqrt{\kappa\ln( \sqrt{\kappa B}/\delta)} \right)
\\ & \qquad +c_1\EE_{x,y}[A(x,y)]2\sqrt{2}c_3\frac{c_7 (\sqrt{2 \log(2) L }+1)(\prod_{l=1}^{L-1} M_l)\sqrt{\sum_{k=1}^D M_{L,k}^2}}{\sqrt{2B}} 
\\ & \qquad + 2 \sqrt{2} c_4 \sqrt{c_5^{2}+c_6^{2}}\frac{c_7 (\sqrt{2 \log(2) L }+1)(\prod_{l=1}^{L-1} M_l)\sum_{k=1}^D M_{L,k}}{\sqrt{m}}
\\ & \qquad +2 \sqrt{2} c_4 \sqrt{c_5^{2}+c_6^{2}}\frac{c_8 (\sqrt{2 \log(2) L' }+1)(\prod_{l=1}^{L'-1} M_l')\sum_{k=1}^D M'_{L',k}}{\sqrt{m}}
\\ & \qquad +c_2 \sqrt{\frac{\ln(2/\delta)}{2m}}
\\ \nonumber & \le  \frac{Q_1}{\sqrt{m}}  +\frac{Q_{2}}{\sqrt{2B}}+c_2 \sqrt{\frac{\ln(2/\delta)}{2m}}
\end{align*}
\end{proof}

\end{document}

%% file: defs.tex
\usepackage[utf8]{inputenc} % allow utf-8 input
\usepackage[T1]{fontenc}    % use 8-bit T1 fonts
\usepackage{microtype}      % microtypography

\usepackage{subcaption}
\usepackage{wrapfig}
\usepackage{float}
\usepackage{chngpage}
\usepackage{enumitem}

\usepackage{graphicx}
\usepackage{mathtools}
\usepackage{microtype}
\usepackage{mathrsfs}
\usepackage{verbatim}
\usepackage{booktabs}
\usepackage{multirow}
\usepackage{blindtext}
\usepackage{mwe}
\usepackage{adjustbox}
\usepackage{algorithm}
\usepackage[font=small,labelfont=bf]{caption}

\usepackage[toc,page]{appendix}
\usepackage{comment}
\usepackage{makecell}

\usepackage{times}
\usepackage{mathtools}
\usepackage{fullpage}
\usepackage{tikz}
\usepackage{xcolor}
\usepackage{todonotes}
\usepackage{footnote}

\newcommand{\eg}{\textit{e.g.}}
\newcommand{\ie}{\textit{i.e.}}

\newcommand{\expected}[1]{\mathbf{E}\mathopen{}\left[ #1 \mathopen{}\right]}
\newcommand{\variance}[1]{\mathbf{Var}\mathopen{}\left[ #1 \mathopen{}\right]}

\def\bb{\vec{b}}

\def\bm{\vec{m}}

\DeclareMathOperator*{\argmax}{argmax}
\DeclareMathOperator*{\argmin}{argmin}

\usepackage{amsmath,amsfonts,bm}

\def\eqref#1{equation~\ref{#1}}
\def\1{\bm{1}}

\def\rb{{\textnormal{b}}}

\def\rve{{\mathbf{e}}}

\def\rvg{{\mathbf{g}}}

\def\vb{{\bm{b}}}

\def\vx{{\bm{x}}}
\def\vy{{\bm{y}}}

\def\mA{{\bm{A}}}

\def\mX{{\bm{X}}}
\def\mY{{\bm{Y}}}

\DeclareMathAlphabet{\mathsfit}{\encodingdefault}{\sfdefault}{m}{sl}
\SetMathAlphabet{\mathsfit}{bold}{\encodingdefault}{\sfdefault}{bx}{n}

\def\sR{{\mathbb{R}}}
\def\sS{{\mathbb{S}}}

\let\ab\allowbreak